\PassOptionsToPackage{inline}{enumitem}
\documentclass[10pt]{article} 

\usepackage[accepted]{rlj} 

\usepackage{amsmath,amsfonts,bm}

\def\eqref#1{equation~\ref{#1}}

\def\1{\bm{1}}

\DeclareMathAlphabet{\mathsfit}{\encodingdefault}{\sfdefault}{m}{sl}
\SetMathAlphabet{\mathsfit}{bold}{\encodingdefault}{\sfdefault}{bx}{n}

\newcommand{\E}{\mathbb{E}}

\newcommand{\R}{\mathbb{R}}

\DeclareMathOperator*{\argmax}{arg\,max}

\usepackage{amssymb}
\usepackage{mathtools}
\usepackage{mathrsfs}

\usepackage{graphicx}
\usepackage{subcaption}
\usepackage[space]{grffile}

\usepackage{booktabs} 
\usepackage{multirow}
\usepackage{amsthm}
\usepackage{todonotes}

\newcommand{\D}{\mathcal{D}}

\newcommand{\sts}{\mathcal{S}} 
\newcommand{\as}{\mathcal{A}} 

\newcommand{\bq}{\mathbf{q}}

\newtheorem{proposition}{Proposition}
\newtheorem{corollary}{Corollary}

\title{Q-based Variational Inverse Reinforcement Learning}
\setrunningtitle{Q-based Variational IRL}

\author{Ondrej Bajgar\textsuperscript{1,2}, Peter Tisnikar\textsuperscript{1}, Alessandro Abate\textsuperscript{1}, Konstantinos Gatsis\textsuperscript{3}, and Maike Osborne\textsuperscript{1}}

\emails{ondrej@bajgar.org}

\affiliations{
$^{1}$\textbf{University of Oxford}\\
$^{2}$\textbf{Artificial Intelligence Center, Czech Technical University in Prague}\\
$^{3}$\textbf{Villanova University}\\
}


\contribution{
We introduce Q-based Variational IRL (QVIRL), a method for Bayesian IRL that is scalable and provides principled and well-calibrated posterior distribution approximations over both the optimal Q-values as well as the rewards.
}{
The only scalable method for Bayesian IRL, AVRIL \citep{chan2021}), does not provide a full posterior distribution over optimal Q-values (learning only a point estimate of the optimal Q-values), and those that do \citep{ramachandran2007, mandyam2023, bajgar2024} cannot scale to continuous and large state spaces due to relying on Monte Carlo sampling of the reward posterior. We instead turn to explicit optimal Q-value posterior estimation using variational inference, thus creating a method that is both principled and scalable.
}

\contribution{
We provide an extensive analysis of the quality of the recovered posterior in randomized gridworlds, where we can obtain a true posterior over rewards for comparison. We show that QVIRL's Gaussian variational posterior produces a faithful and well-calibrated posterior relative to AVRIL, which exhibits, depending on a hyperparameter value, either variance collapse and over-fitting, or reversion to the prior.
}{Having a faithful posterior approximation enables performance on downstream tasks including risk-averse apprenticeship learning and active learning, which both build on uncertainty quantification.}

\contribution{ 
In apprenticeship learning settings, QVIRL can learn policies that reach or even surpass expert performance, even when given only few expert demonstrations. 
In active learning settings, QVIRL provides a useful knowledge acquisition signal, and drives successful active learning.
}{
While these properties have been available through Bayesian IRL methods based on MCMC sampling \citep{ramachandran2007,bajgar2024}, these were mostly limited to tabular or linear settings. While ValueWalk was applied to Lunar Lander, it can take over 12 hours to train, where QVIRL takes minutes. No prior method for Bayesian IRL is both easily scalable beyond tiny settings and provides high-quality uncertainty estimation.
}

\summary{
The development of safe and beneficial AI requires that systems can learn and act in accordance with human preferences. However, explicitly specifying these preferences by hand is often infeasible. Inverse reinforcement learning (IRL) addresses this challenge by inferring preferences, represented as reward functions, from expert behaviour. We introduce Q-based Variational IRL (QVIRL), a new Bayesian IRL method that recovers a posterior distribution over rewards from expert demonstrations via primarily learning a variational distribution over optimal Q-values. Unlike previous approaches, QVIRL combines scalability with uncertainty quantification, enabling risk-averse planning as well as active learning. We demonstrate QVIRL's strong performance in apprenticeship learning across various tasks, including gridworlds, Lunar Lander, the Highway Environment, and two ATARI games both with static expert data and with active learning. It is the first work in Bayesian IRL that demonstrates training from raw pixel observations.
}

\begin{document}

\maketitle  

\begin{abstract}
The development of safe and beneficial AI requires that systems can learn and act in accordance with human preferences. However, explicitly specifying these preferences by hand is often infeasible. Inverse reinforcement learning (IRL) addresses this challenge by inferring preferences, represented as reward functions, from expert behaviour. We introduce Q-based Variational IRL (QVIRL), a novel Bayesian IRL method that recovers a posterior distribution over rewards from expert demonstrations via primarily learning a variational distribution over optimal Q-values. Unlike previous approaches, QVIRL combines scalability with uncertainty quantification, important for safety-critical applications as well as active learning. We demonstrate QVIRL's strong performance in apprenticeship learning across various tasks, including gridworlds, Lunar Lander, the Highway Environment, and two ATARI games both with static expert data and with active learning. It is the first method for Bayesian IRL that demonstrates training from raw pixel observations.
\end{abstract}

\section{Introduction}

Stuart Russell \citeyearpar{zotero-1078} suggested three principles to guide the development of beneficial AI: AI's only objective is realizing human preferences; AI is initially uncertain about what these preferences are; and the ultimate source of information about the preferences is human behaviour. \emph{Apprenticeship learning}\footnote{We use the term apprenticeship learning to mean a subarea of imitation learning that uses IRL as an intermediate step. Imitation learning is thus a broader term including non-IRL techniques for learning a good policy from demonstrations, such as behavioural cloning.} via Bayesian \emph{inverse reinforcement learning} (IRL) can be seen as operationalizing these principles: Bayesian IRL starts with a prior distribution over reward functions representing initial uncertainty about human preferences.
It then combines this prior with \emph{demonstration} data from a human expert acting approximately optimally with respect to an unknown reward, to produce a posterior distribution over rewards. In apprenticeship learning, this posterior over rewards is then used to produce a policy that should perform well according to the unknown reward function.

Non-Bayesian IRL has been successfully applied in apprenticeship learning in settings including robotics \citep{kretzschmar2016,okal2016,woodworth2018,das2021,liu2022a}, navigation in Google Maps on a global scale \citep{barnes2023}, and autonomous driving \citep{sun2018,rosbach2019,huang2022}, including on vehicles deployed in real-world heavy traffic \citep{phan-minh2022}. However, this body of work that has scaled IRL to real-world settings generally learns only a point estimate of the reward function, which can be problematic for two main reasons: first, the IRL task is generally underspecified -- there exist multiple reward functions that equally well explain given behaviour, even in the limit of many demonstration trajectories \citep{russell1998,cao2021,kim2021,skalse2023}. Second, further uncertainty is induced by working only with a limited set of demonstrations. 

\textit{Bayesian} IRL addresses this by recovering a posterior distribution over reward functions
rather than a point estimate, which is useful to produce policies robust with respect to this posterior uncertainty, or to do active learning and gather more data reducing the uncertainty where it matters. However, prior methods for Bayesian IRL (see Section~\ref{sec:rw}) are generally applicable only to finite state and action spaces, or continuous spaces with only a handful of dimensions \citep{ramachandran2007,mandyam2023,bajgar2024}, thus hindering their applicability to real-world settings. Alternatively, other work \citep{chan2021} forgoes crucial parts of posterior uncertainty quantification in the interest of scalability.

The contribution of this paper is providing a method that preserves both scalability -- in terms of the dimensionality of the state or observation space as well as the amount of demonstration data, which we demonstrate by applying it to higher dimensional settings than any previous work on Bayesian IRL -- and full posterior uncertainty estimation, whose usefulness we demonstrate in active learning. We achieve this by leveraging variational inference (VI), which is notably more computationally efficient than Markov chain Monte Carlo used in most previous work, and by primarily working in the space of optimal Q-values rather than rewards, which avoids the need to repeatedly solve the forward planning problem -- a notorious obstacle in IRL. Applying Gaussian VI faces three obstacles in IRL: firstly, the posterior needs to be consistent under the Bellman equation, but the Gaussian family is not closed under the maximization operation the equation contains. Secondly, the likelihood -- a softmax over jointly Gaussian random variables -- is not known analytically. Finally, the KL divergence to the prior is not easy to evaluate over continuous domains. We offer and test approximations that address these challenges. We see this as a step in scaling Bayesian IRL methods to higher-dimensional settings, broadening the range of tasks that can benefit from the improved robustness brought by posterior uncertainty quantification, and opening the door to further work on scaling, robust apprenticeship learning, and active IRL.


\section{Problem Formulation}
\label{sec:task}

The goal of Bayesian IRL is to recover a posterior distribution over reward functions based on observing a set of demonstrations $\D=\{(\phi(s_1),a_1),...,(\phi(s_n),a_n)\}$ from an expert acting in a 
 Markov decision process (MDP) $\mathcal{M}=\left(\sts, \as, p, r, \gamma, \rho_0\right)$ where $\sts$ and $\as$ are the state and action spaces respectively,
$\phi: \sts \to \Phi$ is a feature function that maps each state to its representation in a feature space $\Phi$, $p: \sts \times \as \to \mathcal P(\sts)$ is the transition function, where $\mathcal P(\sts)$ is a set of probability 
measures over $\sts$, $r: \Phi\times \as \to \mathbb{R}$ is an (expected) reward function\footnote{\label{footnote:stochastic-rewards}Our formulation permits the underlying reward to be stochastic. Our expert model makes decisions using the optimal Q-function, which depends only on the expectation of the reward and therefore the observed demonstrations only give us information about the expectation. Because of this, the learnt reward function can represent either deterministic reward or an expectation of a stochastic one.},
$ \gamma\in (0,1) $ is a discount rate, and
 $\rho_0$ is the initial state distribution. Where there is no risk of confusion, we sometimes write rewards, Q-functions, and policies directly as functions of the state, omitting the feature function.

In IRL, we know all elements of the MDP except for the reward and, possibly, the transition function.
Instead, we have a model of how the expert policy is linked to the reward and, in the case of Bayesian IRL, we also have a prior distribution over reward functions. In this work we will assume that conditional on the optimal Q-values, the actions chosen by the expert are independent, and use the standard Boltzmann rationality model to describe the expert's action likelihood across a set of demonstrations $\D$:
\begin{equation} 
\label{eq:boltzmann-likelihood} p(\D|r) = \prod_{(s_t,a_t)\in\D} \frac{e^{\beta Q^*(\phi(s_t),a_t)}}{\sum_{a'\in\as} e^{\beta Q^*(\phi(s_t),a')}} \end{equation} 
\citep{ramachandran2007,chan2021,bajgar2024}. The optimal Q-value $Q^*(s,a)$ is the expected discounted return if action $a$ is taken in state $s$ and the optimal policy $\pi^*$ is subsequently followed, and $\beta \in (0, \infty)$ is a rationality coefficient.

Given this likelihood together with the prior over rewards $p(r)$, we can calculate the posterior using Bayes' theorem as $p(r|\D) = p(\D|r)p(r)/p(\D)$. The full probability of the demonstrations would involve the transition probabilities and the initial state distribution, but since the same factors appear in both the numerator and the denominator of the posterior, they cancel, so we omit them from the likelihood. Except for a few special cases, we cannot calculate the posterior analytically and need to resort to approximate methods. We propose to use variational inference, in a manner that substantively improves upon the only previous use of variational inference for Bayesian IRL.

We then use this Bayesian IRL method for apprenticeship learning, where the goal is to produce an \emph{apprentice policy} $\pi_{\text{A}}$ that performs well under the unknown reward function (either in terms of expected return, or a risk-averse criterion such as CVaR), first with static data, and then also with active learning, where, in each active step, we select an initial state for the next expert demonstration, so that our posterior uncertainty is maximally reduced, or the apprentice performance maximally improved.

\section{Related Work} \label{sec:rw}
Inverse reinforcement learning (IRL) was introduced by \citet{russell1998}, though essentially the same problem had been formulated and studied before as \emph{inverse optimal control} (\citet{kalman1964};  
the communities studying each of these two formulations have been somewhat separate and using different sets of methods; see \citet{abazar2020} for a comparison). The already vast IRL literature is well reviewed by \citet{arora2021} and \citet{adams2022} -- here we concentrate on approaches closest to ours, i.e. \textit{Bayesian} inverse reinforcement learning methods.

The problem of \textit{Bayesian} IRL was first addressed by \citet{ramachandran2007} using Markov chain Monte Carlo (MCMC) sampling. MCMC can produce samples from the true posterior over rewards, but scales poorly to higher dimensions. Furthermore, Ramachandran's method needs to solve the forward RL problem many times. \citet{michini2012a} improved the efficiency by focusing only on the most relevant parts of the state-action space. \citet{mandyam2023} instead tried to reduce the number of times the forward RL problem needs to get solved by solving it repeatedly with different reward functions and learning a joint density over rewards and demonstrations using kernel density estimation, but this amortization method can still only address very small problems. \citet{bajgar2024} significantly reduced the computational burden by performing inference primarily in the space of optimal Q-values, thus avoiding the need to repeatedly solve the forward RL problem -- a trick we also use. However, the method is still based on MCMC sampling, which is computationally intensive and has trouble scaling to higher-dimensional inference problems.

\citet{chan2021} have applied the much more scalable variational inference to the problem. Their approach avoids having to repeatedly solve the forward RL problem by jointly learning a reward encoder (a network producing a 
mean and variance for the reward in any state) and a Q-network, capturing the current estimate of the optimal Q-function. The two networks are tied together by the Bellman equation applied as a soft constraint. The Q-network is then used to 
produce the apprentice policy. However, since the method models only a point estimate of the Q-function, it does not provide any uncertainty estimation for the apprentice policy. Also, since the reward posterior is tied to the Q-function point estimate and thus a single policy, its posterior variance is greatly reduced relative to the true Bayesian posterior, which needs to account for optimal policy uncertainty. The lack of uncertainty over Q-values and policies also makes the method unusable for existing active IRL methods. By contrast, our method directly models the uncertainty over optimal Q-values which in turn cheaply induces the posterior over rewards, so we keep full uncertainty quantification over both.

\section{Method}
\label{sec:method}

\begin{figure}[t]
    \centering 
    \includegraphics[width=0.9\textwidth, trim={1cm 8cm 2.5cm 1cm},clip]{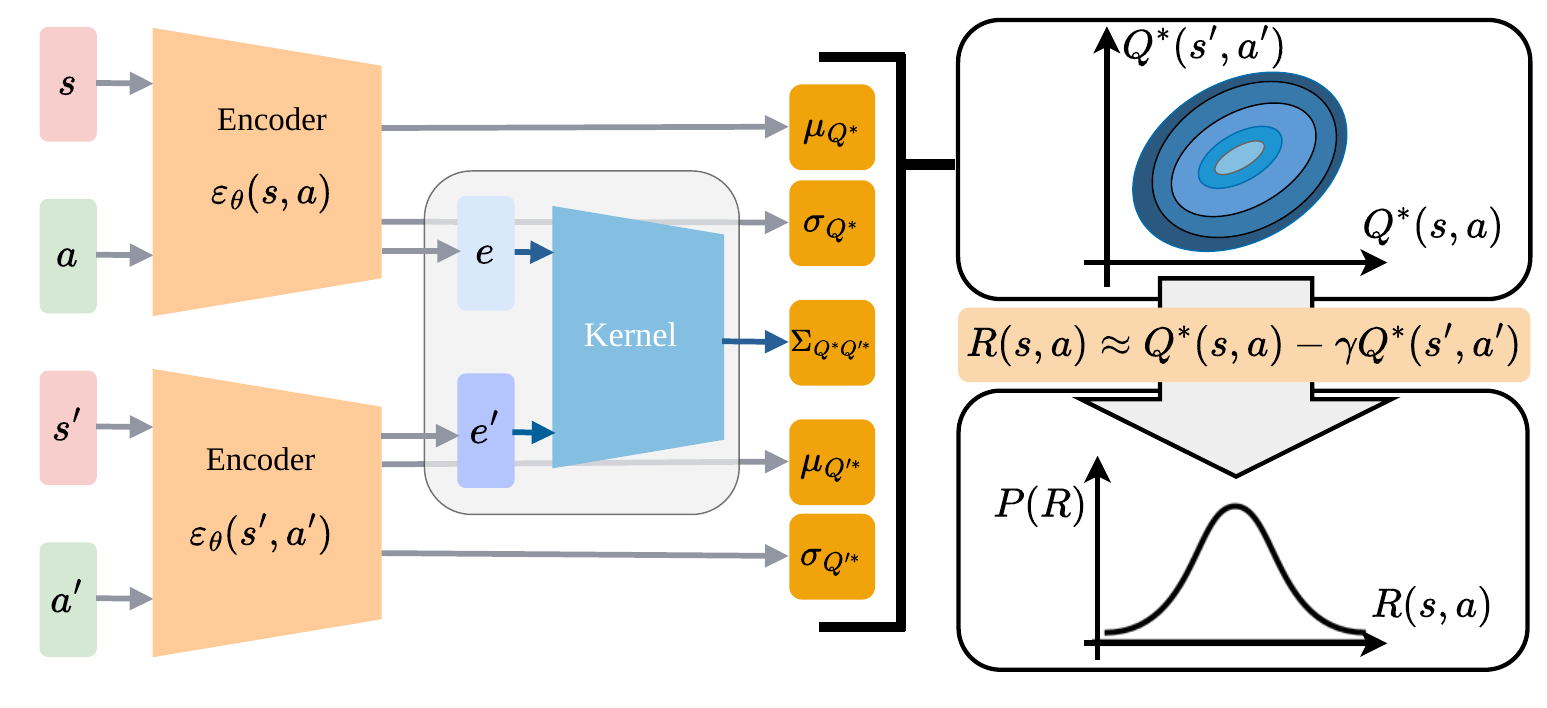} 
    \caption{Illustration of the QVIRL architecture. A state-action pair $(s,a)$ and a successor pair $(s',a')$ are fed into an encoder network that outputs the mean $\mu_{Q^*}$ and the standard deviation $\sigma_{Q^*}$ of the associated optimal Q-value $Q^*(s,a)$, as well as a latent embedding $e$. The embeddings of a set of state-action pairs can be passed to a kernel to produce their correlation matrix and, combined with individual variances, their covariance $\text{Cov}(Q^*(s,a),Q^*(s',a'))$ denoted by $\Sigma_{Q^*Q'^*}$ in the diagram. Together, they define the joint variational distribution of the optimal Q-values of a given set of state-action pairs. This distribution can then be used to deduce the distribution over the corresponding reward through the Bellman equation.} 
    \label{fig:architecture}
\end{figure}

We first provide a high-level summary of our method, \emph{Q-based Variational IRL} (QVIRL), and then proceed by describing each of its components in detail. Our method takes as input a set of demonstrations and a prior over reward functions, and outputs a variational approximation to the posterior over optimal Q-functions. The posterior can then directly be used to deduce a variational distribution over reward functions by being transformed through the inverse (optimal) Bellman operator. The architecture is illustrated in Fig.~\ref{fig:architecture}. 

To efficiently compute the reward posterior, we build on an insight used already by \citet{chan2021}, \citet{garg2021}, and \citet{bajgar2024}: going from Q-values to the reward is computationally much simpler than going from rewards to the optimal Q-function. Our model therefore centres on modelling a posterior distribution over optimal Q-values. This Q-value posterior can be used to evaluate the likelihood (Eq.~\ref{eq:boltzmann-likelihood}), and the Bellman equation allows us to deduce the corresponding posterior over rewards in one step (while the opposite would require pushing the reward posterior through a costly Q-learning algorithm). We approximate the posterior over optimal Q-values by a variational distribution optimized using and objective similar to the evidence lower bound (ELBO).

We now describe each component in turn. We present the method in a form that applies to both discrete and continuous state spaces by expressing the dependence on the environment dynamics as an expectation over successor states. This expectation can be calculated exactly in the tabular case with known dynamics. In continuous environments, we approximate this by samples coming from an expert trajectory or an auxiliary trajectory collected from an online apprentice policy. Further details and other options are provided in Appendix~\ref{app:method-details}.

\subsection{Variational Posterior over Optimal Q-Values}

We model the posterior distribution over optimal Q-values using a variational family $q_{\theta}$, parameterized by a vector of variational parameters $\theta\in\R^{d}$, which we can use to evaluate the joint distribution of optimal Q-values for any given set of state-action pairs. While a range of variational distributions can be used, we use a multi-variate Gaussian (in discrete state spaces) or a Gaussian process (in continuous spaces).\footnote{An extension beyond a strictly Gaussian case can be achieved through warping \citep{snelson2003}. For instance, the sinh-arcsinh transformation~\citep{jones2009} can be useful to model skewness and kurtosis.}

We model the covariance between the Q-values explicitly. Modelling optimal Q-value covariances is crucial as demonstration likelihoods are invariant to constant shifts in Q-values, and Bellman dependencies couple Q-values along trajectories so the posterior uncertainty over optimal Q-values is generally correlated. This is a fundamental property of MDPs, so a good method for modelling uncertainty over optimal Q-values should be able to model such correlation structure.

The variational posterior architecture that we will use for our experiments, and which is outlined in Figure~\ref{fig:architecture}, consists of a neural-network encoder $\varepsilon_\theta$ with parameters $\theta$, which, for any given state-action pair $(s,a)$, produces a mean $\mu_{Q^*}(s,a;\theta)$, a log standard deviation $\log \sigma_{Q^*}(s,a;\theta)$, as well as a latent-space embedding $e(s,a;\theta)$. The embeddings can be used as an input to a Gaussian process kernel $k$ to calculate the correlation matrix of any set of state-action pairs 
(we will be using a radial-basis function (RBF) kernel). For any collection $(s_1, a_1), \dots, (s_n, a_n)$ of state-action pairs, the variational posterior will yield a joint multivariate Gaussian distribution with 
mean and covariance 
\begin{equation}
\label{eq:q-mean-cov}
	\begin{split}
       	\boldsymbol{\mu}_{Q^*} & =\big[\mu_{Q^*}(s_i,a_i;\theta)\big]_{i=1}^n \\ 
        \Sigma_{Q^*} & = \big[ \sigma_{Q^*} (s_i, a_i;\theta) \sigma_{Q^*}(s_j, a_j;\theta) k\big(e(s_i, a_i; {\theta}), e(s_j, a_j; {\theta})\big) \big]_{i,j=1}^n.
	\end{split}
\end{equation}

\subsection{From Optimal Q-Values to Rewards} \label{ssec:implicit-reward}

In a discounted MDP, there is a bijection between the expected reward function
and the optimal Q function. 
Thus, the variational posterior that we have just introduced implicitly fully defines a posterior over rewards. In particular, according to the inverse optimal Bellman equation, 
\begin{equation}
\label{eq:bellman-reward}
	R(s,a) = Q^*(s,a) - \gamma \,\E_{s'\sim p(\cdot|s,a)}\!\left[\max_{a'} Q^*(s',a')\right].
\end{equation}
Unfortunately, the final term -- the optimal state value $V^*(s')=\max_{a'}Q^*(s',a')$, a maximum of jointly Gaussian random variables -- cannot be evaluated analytically to yield a closed-form distribution for the reward. One can obtain an empirical reward posterior sample via Monte Carlo sampling. This is feasible especially at inference time. During training, one can use the reparameterization trick to propagate gradient through those samples; however, this can be computationally costly or substantially increase gradient variance. 

We propose to address these issues by applying Clark's (\citeyear{clark1961}) method. In contrast to the general case of $n$ Gaussians, the first two moments of the maximum of a \emph{pair} of jointly Gaussian variables are available in closed form, and the maximum over $n$ variables (actions) can be built up by taking such pairwise maxima one at a time, approximating each intermediate maximum by a Gaussian, to yield a final approximate distribution $V^*\sim \mathcal{N}(\mu_V,\sigma^2_V)$. The full equations for this standard approximation are provided in Appendix~\ref{app:clark}. 
Though approximate, the Clark approximation gives a Gaussian reward posterior in closed form, which can be used to calculate the KL divergence to the prior significantly more efficiently than by random sampling with the reparameterization trick.

As a simpler alternative, we also tested a \emph{max-mean} approximation, approximating the distribution of the maximum by that of the highest-mean argument and found it to also work well in most cases. A comparison is provided in the appendix, but the results in the main text are using the Clark approximation. 

Once we apply either approximation, the reward in Eq.~\ref{eq:bellman-reward} becomes a linear function of jointly Gaussian random variables and is thus normally distributed with mean 
$\mu_R(s,a;\theta)=\mu_Q(s,a;\theta)-\gamma\E_{s'}\mu_V(s';\theta)$
and variance
\begin{equation}
\label{eq:reward-std}
    \sigma^2_R(s,a;\theta) = \sigma^2_Q(s,a;\theta)
    - 2\gamma\,\E_{s'}\!\big[\mathrm{Cov}\big(Q^*(s,a), V^*(s')\big)\big]
    + \gamma^2\,\E_{s',s''}\!\big[\mathrm{Cov}\big(V^*(s'), V^*(s'')\big)\big],
\end{equation}
where $s'$ and $s''$ are independent draws from $p(\cdot|s,a)$, and the covariance of the approximate value $V^*$ with the current optimal Q-value and across successor states follows from the joint Gaussian of Eq.~\ref{eq:q-mean-cov} under the Clark approximation (Appendix~\ref{app:clark}). When the expectation is estimated by a single successor sample (our default in continuous environments), the last two terms reduce to $-2\gamma\,\mathrm{Cov}\big(Q^*(s,a),V^*(s')\big) + \gamma^2\sigma^2_V(s';\theta)$.

\subsection{Likelihood}


We model demonstrator actions using a Boltzmann policy induced by the uncertain optimal Q-values. If the posterior is approximated using a variational density $q_\theta$, then the probability of the demonstrations under the posterior can be written as 
$p(\D|\theta)=\prod_{(s,a)\in\D}p(a|s;\theta)$ with 
\begin{equation}
\label{eq:likelihood}
    p(a|s;\theta) = \int_{Q^*(s,\cdot)\in\R^{|\as|}} \frac{e^{\beta Q^*(s,a)}}{\sum_{a'}e^{\beta Q^*(s,a')}} q_\theta(Q^*(s,\cdot)) \; d Q^*(s,\cdot)
\end{equation} 
where by $Q^*(s,\cdot)$ we denote the vector of optimal Q-values for all actions in state $s$.

The above integral is not known to have a closed-form solution; however, following the approximation proposed by \citet{lu2021},\footnote{
An alternative would be to use stochastic variational inference with the reparameterization trick~\citep{kingma2022}, but in preliminary experiments it performed worse not only in training speed (as expected) but also in apprentice performance. However, due to the very slow speed, this has not been tuned extensively.} we can approximate Eq.~\ref{eq:likelihood} with
\begin{equation}
    \label{eq:mean-field-action-prob}
    p(a|s;\theta)\approx \scalebox{0.925}{$\left(\sum_{a'} \exp\left(- \frac{\beta\left(\mu_{Q^*}(s,a) - \mu_{Q^*}(s,a')\right)}{\sqrt{1+ 3\pi^{-2} \beta^2 (\sigma_{Q^*}(s,a)^2+\sigma_{Q^*}(s,a')^2-2\Sigma_{aa'}) }}\right) \right)^{-1}$}
\end{equation}
where $\Sigma_{aa'}$ is the covariance between the optimal Q-values of $a$ and $a'$ as calculated via the embeddings and the kernel as $\Sigma_{aa'}:=\sigma_{Q^*}(s,a)\sigma_{Q^*}(s,a')k(e(s,a),e(s,a'))$. 

\subsection{Training Objective}

We will optimize the parameters $\theta$ of our encoder network $\varepsilon_\theta$ using stochastic gradient descent on the loss
\begin{equation}
\label{eq:elbo}
    \scalebox{0.92}{$\mathcal{L}(\theta) = - \sum_{s,a\in \D}[\log p(a|s; \theta)] + \text{KL}(q_R(R;\theta) || p(R)),$}
\end{equation} 
where the first term is the log-likelihood (Eq.~\ref{eq:mean-field-action-prob}), $q_R(R;\theta)$ is the implicit variational posterior over reward (Section~\ref{ssec:implicit-reward}), and $p(R)$ is the prior distribution over reward. In a tabular setting with a multivariate Gaussian prior, the KL can be calculated analytically. In continuous settings, we evaluate the KL between multivariate Gaussians on both the expert transitions within each mini-batch, and, if available, also on a set of auxiliary states, which can come, for instance, from rollouts of the current apprentice policy or from an offline set of non-expert trajectories. Full details are provided in Appendix~\ref{app:kl}. Note that the posterior is trained only where data are provided. While we do provide also some offline results with only expert data (mostly for fair comparison to prior work), uncertainty quantification is useful especially in parts of the state-space uncovered by the expert, so training with some form of auxiliary data is the recommended setup.

\section{Experiments} 
\label{sec:experiements}

We evaluate QVIRL along three dimensions corresponding to the main claims of the paper. First, in randomized gridworlds where high-quality MCMC samples from the reward posterior are available, we test whether QVIRL provides an accurate and calibrated approximation to the Bayesian posterior. Second, we evaluate whether this posterior is useful for apprenticeship learning, both when extracting policies from the posterior mean and when using posterior uncertainty to derive risk-sensitive policies. Third, we test whether QVIRL can serve as a scalable Bayesian backbone for active inverse reinforcement learning, where acquisition functions require uncertainty estimates over Q-values or policies.

We conduct our experiments in three types of environments. Randomized gridworlds provide a controlled tabular setting with known transition dynamics, allowing us to compare variational posterior approximations to ValueWalk~\citep{bajgar2024}, a MCMC-based method able to sample from the true posterior, which we use as a reference oracle. Lunar Lander and Highway Env provide larger continuous-state benchmarks for evaluating apprenticeship learning performance, stability, and data efficiency from compact vector-based observations. Finally, we run apprenticeship learning on ATARI games with raw-pixel observations to demonstrate the scaling potential of our method. We provide further details on the environments, expert policies, and demonstration collection procedure in Appendix~\ref{app:exp-details}.

We compare QVIRL against three baselines. AVRIL~\citep{chan2021} is the closest prior scalable Bayesian IRL method based on variational inference. IQ-Learn~\citep{garg2021} is a strong imitation-learning baseline that learns a point estimate of the soft Q-function; we also include an ensemble version to test whether ensembling alone provides comparable uncertainty benefits. Behaviour cloning~\citep{pomerleau1991} serves as a simple supervised imitation-learning baseline.

\subsection{Evaluation of Posterior Approximation Quality} \label{ssec: post-quality}

To evaluate the quality of the posterior approximation, we run QVIRL and AVRIL on $100$ random 8$\times$8 gridworlds with Gaussian per-state rewards and randomly distributed terminal states. For each gridworld, we generate $5$ expert demonstrations of length $5$ with a rationality coefficient of $\beta=2$. We then use ValueWalk to get $5000$ samples from the true reward posterior and evaluate how closely the two variational methods, AVRIL and our QVIRL, approximate this true posterior.


Table~\ref{tab:gridworld-results} reports the posterior quality metrics. According to each of them, QVIRL produces a  higher-fidelity posterior than AVRIL -- the posterior log density and CI calibration even matches that of ValueWalk  (where the $\log p$ is estimated by kernel density estimation (KDE)). On the other hand, AVRIL's posterior is much further from ground truth and its $90\%$ credible intervals show overconfidence. We confirm this by visualising marginal posteriors over rewards for a few example states (Figure~\ref{fig:exp-posteriors}) where we see that AVRIL's posterior is overconfident, while the mean is biased towards the prior distribution. It is also highly sensitive to the value of its $\lambda$ parameter that controls the strength of the soft binding between the rewards and Q values. If this parameter is too low, the posterior reverts to the prior; if too high, the posterior collapses toward a point estimate. The values in the table are for $\lambda$ that minimized $W_1$ to the oracle on a set of 10 held-out seeds (an extra advantage given to AVRIL).

\begin{figure}[tpb]
    \centering
    \includegraphics[width=0.98\linewidth]{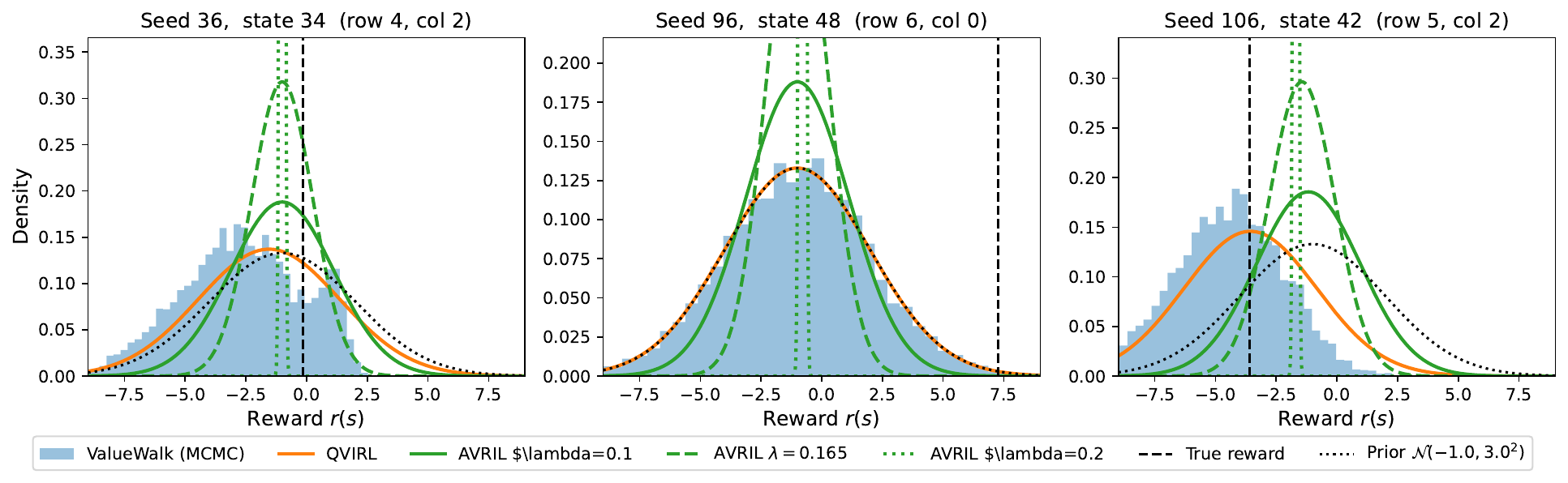}
    \caption{Illustration of posteriors recovered by ValueWalk (MCMC samples used as ground truth), AVRIL (for three values of constraint weight $\lambda)$, and QVIRL in $3$ randomly selected states from our randomized $8\times8$ gridworlds.}
    \label{fig:exp-posteriors}
\end{figure}

\begin{table}[t]
\centering
\small
\caption{Posterior quality on 100 random 8$\times$8 gridworlds (mean $\pm$ stderr across $100$ seeds).
\begin{enumerate*}[label=(\roman*)]
\item log-density of the true reward under the posterior;
\item fraction of states with $r(s)$ within the posterior $90\%$ credible interval;
\item fraction of expert action probabilities $\pi (s | a)$ within the posterior predictive $90\%$ credible interval; and 
\item total variation distance between each method's posterior-predictive policy and that of the ValueWalk oracle.
  \end{enumerate*}
  $\uparrow$/$\downarrow$ indicates higher/lower is better.
  }
\label{tab:gridworld-results}
\resizebox{0.99\textwidth}{!}{
\begin{tabular}{lccccc}
\toprule
Method & (1) $W_1$ (VW) $\downarrow$ &  (2) $\log p(r_\mathrm{true}|\D)$ $\uparrow$  & (3) $r_\mathrm{true}\in 90\%$ CI & (4) $\pi_\mathrm{exp} \in 90\%$ CI & (5) TV (VW) $\downarrow$ \\
\midrule
ValueWalk & $0.000$            & $-2.432 \pm 0.013$ & $0.888 \pm 0.004$ & $0.910 \pm 0.003$ &  $0.000$            \\
QVIRL     & $0.486 \pm 0.011$ & $-2.437 \pm 0.009$ & $0.901 \pm 0.004$ & $0.913 \pm 0.003$ &  $0.207 \pm 0.003$ \\
AVRIL     & $1.704 \pm 0.007$ & $-7.568 \pm 0.284$ & $0.521 \pm 0.006$ & $0.727 \pm 0.005$ &  $0.250 \pm 0.004$ \\
\bottomrule
\end{tabular}
}
\end{table}

\subsection{Apprenticeship Learning}
\label{ssec:app-rew}

In apprenticeship learning settings, we aim to recover a policy given expert demonstrations. We evaluate apprenticeship learning in gridworlds, Lunar Lander, and Highway environments.

Table~\ref{tab:apprenticeship-results} reports policy performance on the \textbf{synthetic gridworlds} from the preceding subsection. Returns are normalised so that 0 corresponds to a random policy and 100 to the optimal policy under the true reward within each random gridworld. The \emph{mean} variant of each method derives a policy by optimising the mean posterior return; the \emph{CVaR} variant optimises the 0.1-CVaR of the return distribution, yielding a risk-averse policy.

\begin{table}[t]
\centering
\caption{Apprenticeship learning on 100 random 8$\times$8 gridworlds (mean $\pm$ stderr). Returns are normalised: 0 = random-policy return, 100 = optimal-policy return under the true reward. $J_{r_\mathrm{true}}(\pi)$: true-reward return of the apprentice policy. $\mathbb{E}_{r|\D}[J_r(\pi)]$: expected return under the ValueWalk oracle posterior. $\mathrm{CVaR}^{0.1}_{r|\D}[J_r(\pi)]$: conditional value-at-risk (10\% level) of the return under the posterior, i.e.\ the expected return in the worst 10\% cases under the ValueWalk reward posterior; the CVaR policy directly optimizes this objective under each method's own posterior.}
\label{tab:apprenticeship-results}
\begin{tabular}{lccc}
\toprule
Policy $\pi$ & $J_{r_\mathrm{true}}(\pi)$ & $\mathbb{E}_{r|\D}[J_r(\pi)]$ & $\mathrm{CVaR}^{0.1}_{r|\D}[J_r(\pi)]$ \\
\midrule
  ValueWalk (mean) & $\phantom{-}75.4 \pm 1.8$ & $112.1 \pm 2.8$ & $\phantom{-}59.9 \pm 2.8$ \\
  ValueWalk (CVaR) & $\phantom{-}75.8 \pm 1.3$ & $107.7 \pm 3.0$ & $\phantom{-}62.1 \pm 2.3$ \\
\midrule
QVIRL (mean) & $\phantom{-}70.8 \pm 2.1$ & $106.4 \pm 3.2$ & $\phantom{-}49.0 \pm 3.9$ \\
QVIRL (CVaR) & $\phantom{-}66.9 \pm 1.9$ & $\phantom{0}89.6 \pm 3.3$ & $\phantom{-}49.7 \pm 2.8$ \\
AVRIL (mean) & $-22.2 \pm 4.8$ & $\phantom{00}5.4 \pm 2.7$ & $-86.3 \pm 3.4$ \\
AVRIL (CVaR) & $\phantom{-}36.7 \pm 1.1$ & $\phantom{0}27.1 \pm 1.4$ & $\phantom{-}14.0 \pm 1.3$ \\
\midrule
Optimal      & $100.0 \pm 0.0$ & $\phantom{0}64.0 \pm 3.4$ & $\phantom{0}-4.1 \pm 4.5$ \\
Random       & $\phantom{-0}0.0 \pm 0.0$ & $\phantom{00}5.4 \pm 1.4$ & $-17.1 \pm 1.3$ \\
\bottomrule
\end{tabular}
\end{table}

QVIRL's policies perform only slightly worse than the ValueWalk oracle under both the true reward and the reward posterior. Optimising CVaR improves worst-case posterior performance at the cost of lower mean return for all methods, although AVRIL remains weak across all settings. QVIRL-CVaR's  CVaR return also lags behind ValueWalk-CVaR's, suggesting lower posterior tail fidelity compared to the posterior mean. Interestingly, the performance of the true-reward-optimal policy evaluated under the posterior is poor, especially in case of CVaR. Since the expert data cover at most 25 states, there are many states with substantial uncertainty where the posterior assigns probability to negative rewards. The optimal policy does not try to avoid such high-uncertainty states.

In the \textbf{continuous-state Lunar Lander and Highway environments}, QVIRL reaches expert-like performance with fewer demonstrations and lower run-to-run variability than the baselines (Figure~\ref{fig:app-return}). All methods eventually achieve expert performance with $15$ or $3$ trajectories respectively, albeit with larger variance than QVIRL. Furthermore, QVIRL's held-out expert-action log-likelihood is also high from a single demonstration, suggesting a better posterior quality, while other methods initially overfit before becoming competitive with more data (Figure~\ref{fig:app-likelihood}). However, all methods start to perform competitively as they see more data.

\begin{figure}[t]
    \centering 
\includegraphics[width=0.48\textwidth]{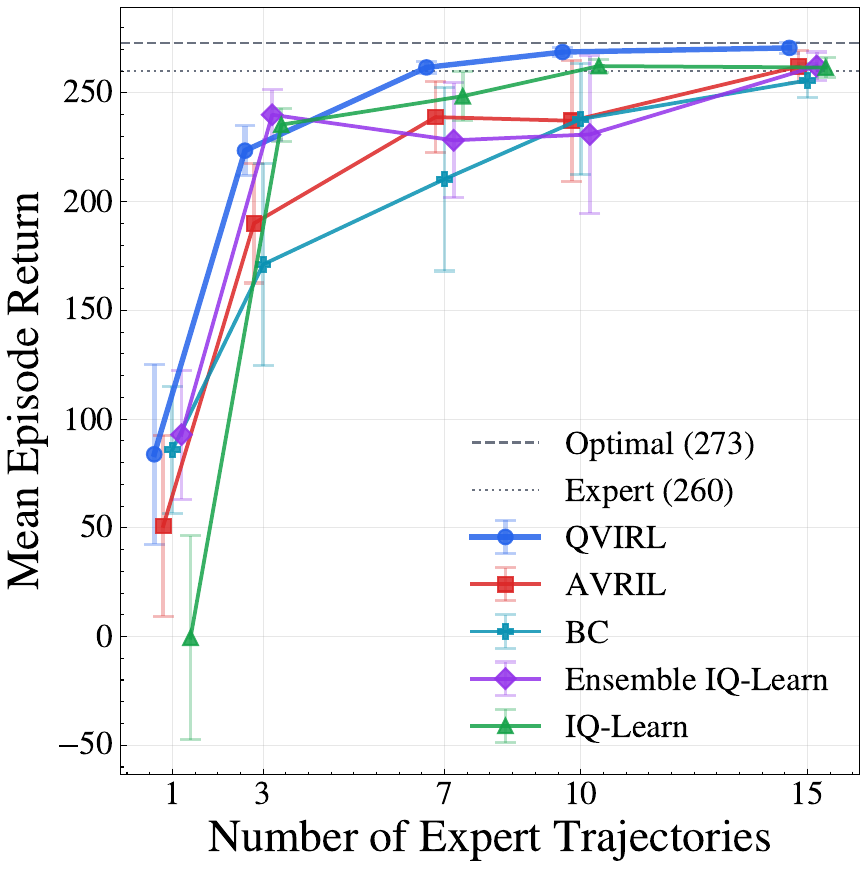}     \includegraphics[width=0.48\textwidth]{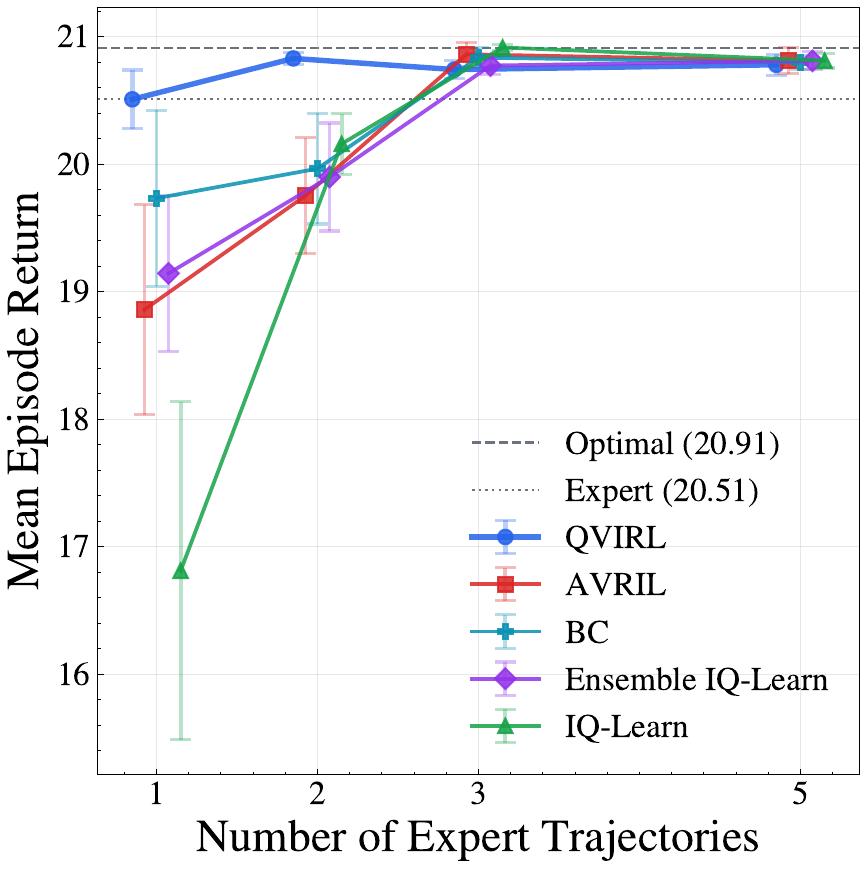} 
    \caption{Mean episode return in apprenticeship learning with different numbers of expert demonstration trajectories in the Lunar Lander (left) and Highway (right) environments. Averaged over $10$ runs, with standard errors of the mean. All methods are evaluated on discrete number of expert trajectories, but we space the points horizontally to improve plot readability.} 
    \label{fig:app-return}
\end{figure}

\begin{figure}[t]
    \centering 
    \includegraphics[width=0.48\textwidth]{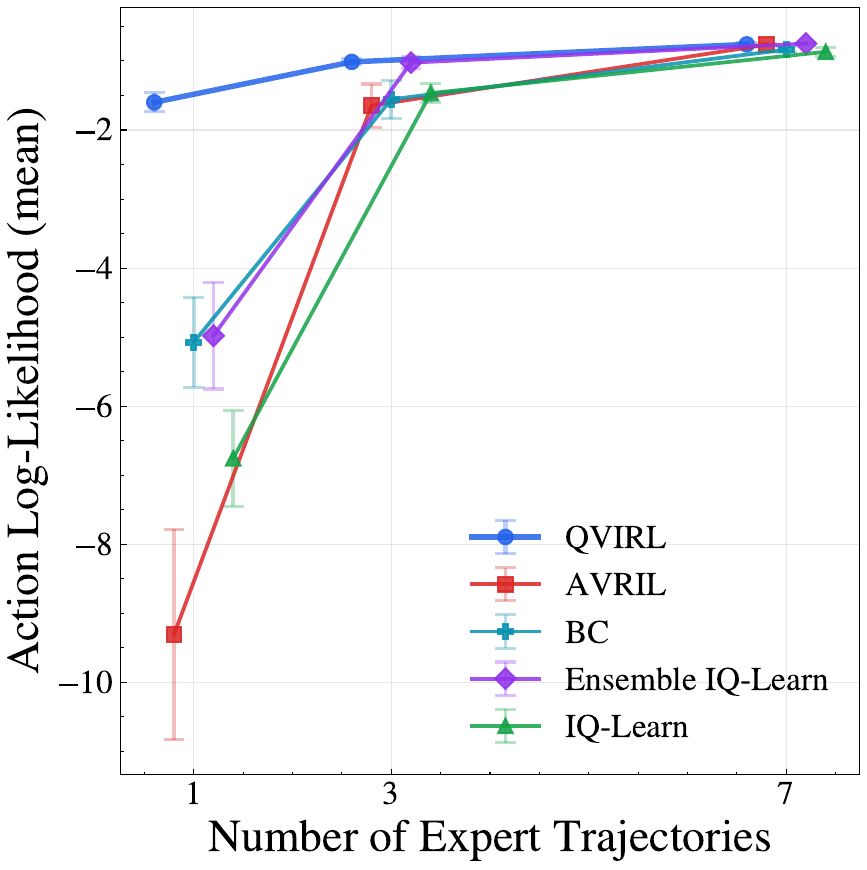}     \includegraphics[width=0.48\textwidth]{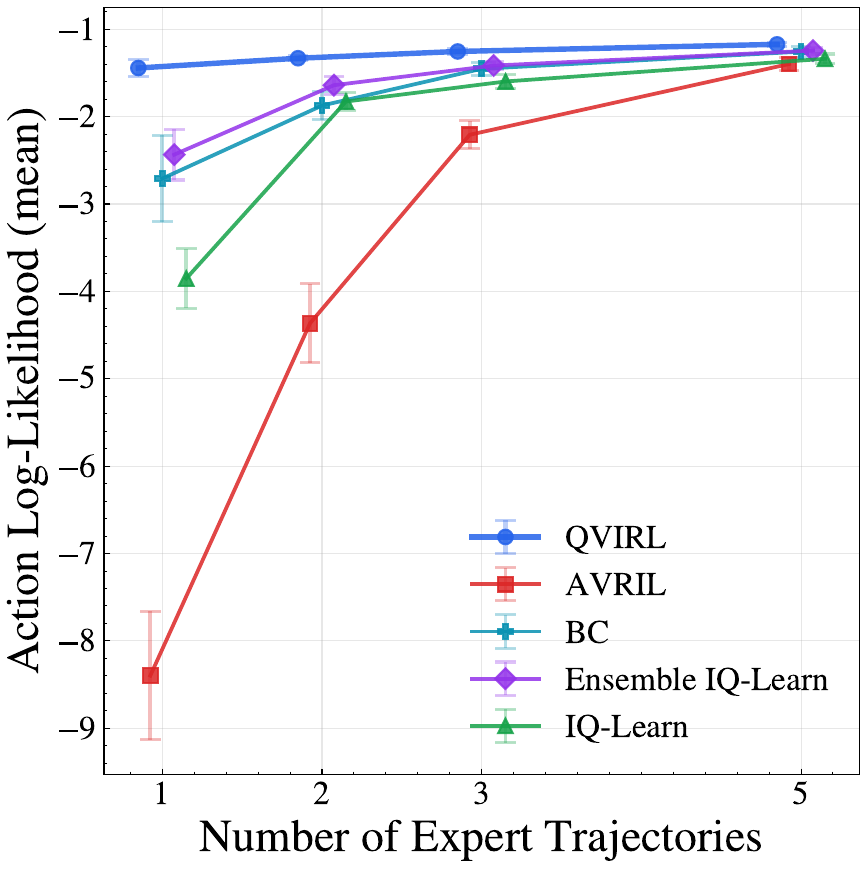} 
    \caption{Mean log-likelihood of held-out expert actions with different numbers of expert trajectories in the Lunar Lander (left) and Highway (right) environments. Averaged over $10$ runs, with standard errors of the mean. All methods are evaluated on discrete numbers of expert trajectories, but we space the points horizontally to improve plot readability.} 
    \label{fig:app-likelihood}
\end{figure}

Finally, we evaluate apprenticeship learning on \textbf{Atari} games from raw-pixel observations, following the evaluation protocol of IQ-Learn~\citep{garg2021} on two of their benchmark games, Pong and Space Invaders. Unlike the vector-observation environments above, QVIRL here uses a convolutional (NatureCNN) encoder trained end-to-end together with the variational Q-value posterior. We train each method on 20 expert demonstration trajectories (matching IQ-Learn's protocol) collected from a Boltzmann-rational expert previously trained using RL, and evaluate the resulting apprentice policy by expected raw game return. Table~\ref{tab:atari-results} shows that QVIRL is competitive with IQ-Learn, the strongest point-estimate imitation-learning baseline, while -- unlike IQ-Learn -- also providing a posterior over rewards. We also evaluate the log-likelihoods on a held-out set of expert demonstrations. While QVIRL is the only method that does proper uncertainty quantification over Q-values, we also provide the corresponding values based on the classification logits from BC and the Boltzmann-rational likelihood using the Q-values learnt by IQ-Learn and AVRIL. These experiments were mainly meant to show that QVIRL's scalability extends to high-dimensional, raw-pixel observation spaces without sacrificing apprenticeship-learning performance relative to a state-of-the-art non-Bayesian method. 

\begin{table}[t]
\centering
\caption{Apprenticeship learning on Atari: mean episode return and log likelihood (LL) of expert actions on held-out demonstrations (mean $\pm$ standard error across seeds). QVIRL LL is computed under the probit likelihood; BC, IQ-Learn, and AVRIL LL are computed under the point likelihood.}
\label{tab:atari-results}
\begin{tabular}{lcccc}
\toprule
 & \multicolumn{2}{c}{Pong} & \multicolumn{2}{c}{Space Inv.} \\
\cmidrule(lr){2-3} \cmidrule(lr){4-5}
Method & Return & LL & Return & LL \\
\midrule
BC       & 19.1 $\pm$ 0.8 & -14.81 $\pm$ 0.49 & 597 $\pm$ 48 & -18.67 $\pm$ 0.05 \\
IQ-Learn & 20.0 $\pm$ 0.2 & -12.43 $\pm$ 2.27  & 740 $\pm$ 31 & -8.99  $\pm$ 1.79 \\
AVRIL    & -5.2 $\pm$ 9.7 & -8.79  $\pm$ 3.80  & 658 $\pm$ 34 & -18.62 $\pm$ 0.07 \\
QVIRL    & 19.7 $\pm$ 0.6 & -2.09  $\pm$ 0.12  & 701 $\pm$ 42 & -5.19  $\pm$ 1.04 \\
\bottomrule
\end{tabular}
\end{table}

\subsection{Active Learning}

Finally, we evaluate whether QVIRL's posterior uncertainty can support active IRL. Most acquisition functions for active IRL require uncertainty estimates over Q-values or (almost equivalently) the unknown expert or optimal policies. MCMC-based Bayesian IRL methods can provide such estimates in small tabular domains, but do not scale to larger continuous-state environments. Conversely, AVRIL provides a reward posterior, but not the Q-value posterior required by existing acquisition functions.\footnote{In our gridworld experiments, we circumvent this issue by sampling rewards from the reward posterior and converting each to the corresponding Q-values via value iteration, but this is very costly and does not translate to environments that don't have known dynamics or are larger.} QVIRL is thus a natural candidate for scaling active IRL beyond tabular settings.

We combine QVIRL with two acquisition functions: (1) Reward EIG~\citep{bajgar2025}, which targets information gain about the reward, and (2) ActiveVaR~\citep{brown2018}, which targets regret-risk reduction. In randomized gridworlds, we compare QVIRL, AVRIL, and ValueWalk using Reward EIG, querying the expert for up to $30$ initial states from which they provide trajectories of length $5$. In Lunar Lander, where no scalable Bayesian oracle is available, we compare only against random sampling. We initialise QVIRL with a single expert trajectory and allow the acquisition function to select initial states from $100$ candidate states sampled from expert data. A trajectory segment of length 10 is provided based on each query.

Figure~\ref{fig:active_results} shows that QVIRL-based acquisition nearly matches ValueWalk in gridworlds, while AVRIL performs significantly worse and stalls after $10$ queries. This suggests that AVRIL's uncertainty estimates do not update sufficiently to support informative querying and thus obtaining new informative data. In Lunar Lander, QVIRL combined with either Reward EIG or ActiveVaR substantially outperforms random querying, demonstrating that QVIRL can act as a scalable Bayesian backbone for active IRL in settings beyond small tabular domains.

\begin{figure}[t]
    \centering
    \begin{subfigure}{0.45\textwidth}
        \centering
        \includegraphics[width=\linewidth]{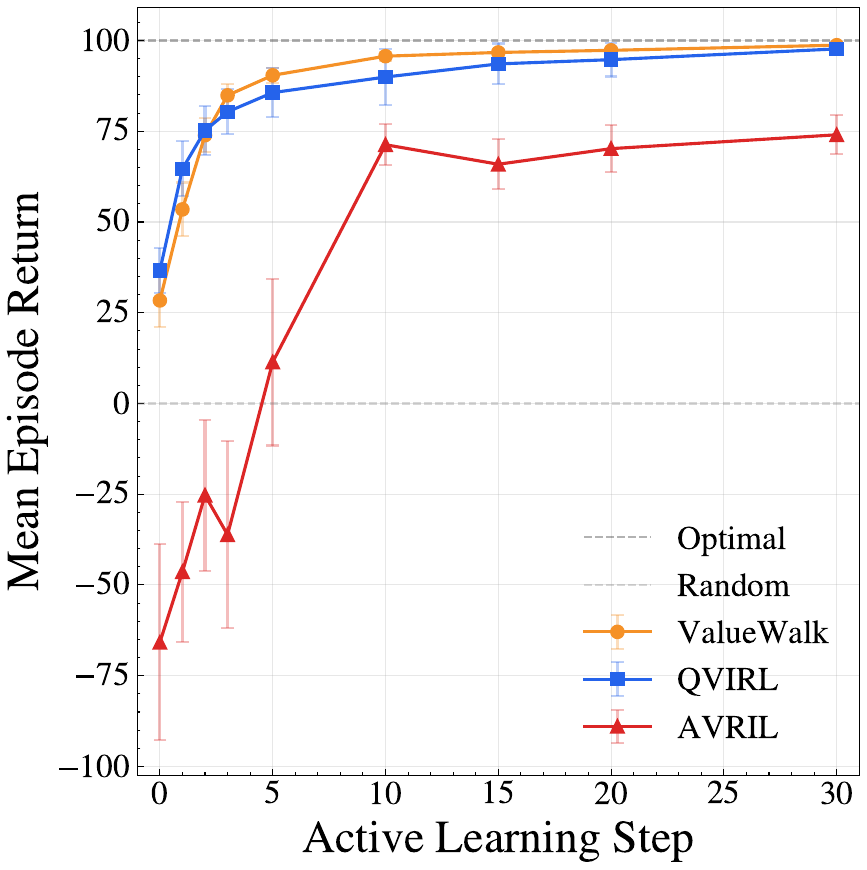}
        \caption{Active learning with QVIRL, ValueWalk, and AVRIL using the Reward EIG acquisition function in the 8x8 randomised gridworlds with demonstration trajectories of length 5. Averaged over $16$ runs, with standard errors of the mean.}
        \label{fig:learning_curve}
    \end{subfigure}
    \hfill
    \begin{subfigure}{0.45\textwidth}
        \centering
        \includegraphics[width=\linewidth]{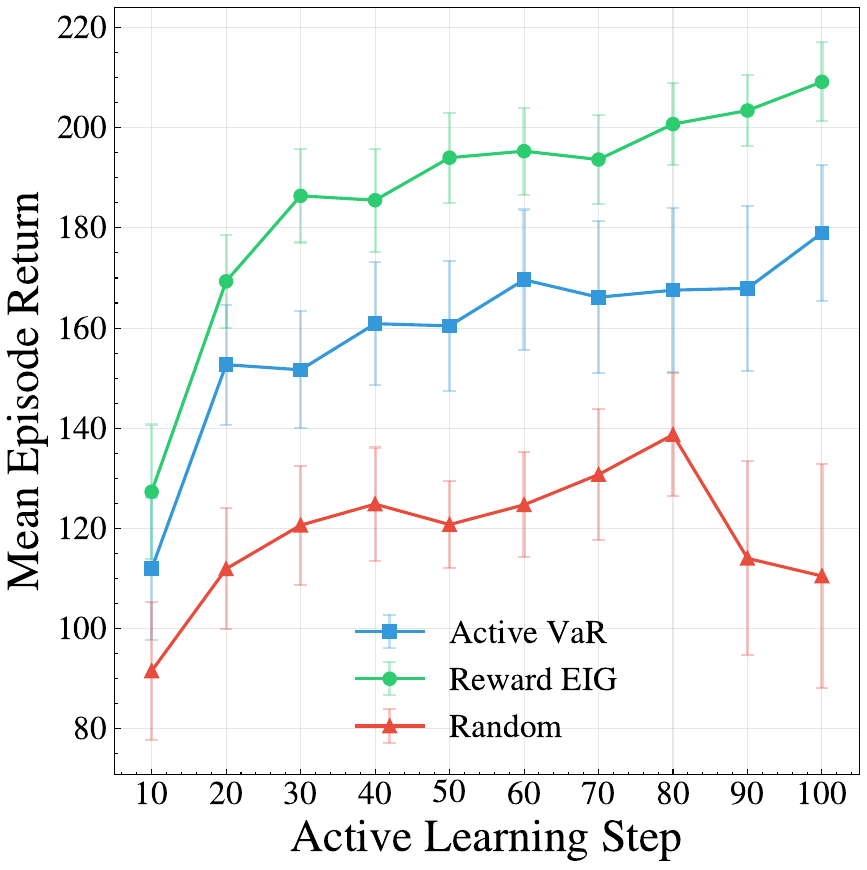}
        \caption{Active learning with QVIRL in Lunar Lander with different acquisition functions, and demonstration trajectories of length $10$. Averaged over $10$ runs, with standard errors of the mean.}
        \label{fig:var_comparison}
    \end{subfigure} 
    \caption{Active learning results on randomized gridworlds (left) and Lunar Lander (right).}
    \label{fig:active_results}
\end{figure}

\section{Discussion} \label{sec: discussion}

In the previous section, we evaluated QVIRL against the desiderata for a Bayesian IRL method.

Firstly, QVIRL aims to approximate the true Bayesian posterior over optimal Q-values and rewards using variational inference. Despite relying on tractable approximations, our empirical results in Section~\ref{sec:experiements} indicate that QVIRL recovers well-calibrated posteriors in tabular environments, unlike AVRIL, which fails to approximate the true reward posterior and only provides a point estimate of the optimal Q-values. MCMC-based Bayesian IRL methods, such as ValueWalk and PolicyWalk, are also capable of estimating the true posterior over the Q-values and rewards but struggle to scale beyond small environments. In contrast, QVIRL leverages variational inference to scale the principled posterior estimation to more complex environments.

Secondly, QVIRL achieves competitive performance in apprenticeship learning from expert demonstrations. It is an efficient, gradient-based algorithm that, unlike most methods for Bayesian IRL, supports training even from raw pixel observations.

Finally, QVIRL naturally enables active learning. By explicitly modelling the optimal Q-value posterior, QVIRL can serve as a scalable Bayesian IRL backbone for acquisition functions that track uncertainty in Q-value space. This removes a key scalability bottleneck of previous MCMC-based active IRL methods and opens the possibility of active learning in much larger environments.

\section{Conclusion}

We have introduced a scalable method for Bayesian inverse reinforcement learning that preserves posterior uncertainty quantification over both optimal Q-values and rewards. To our knowledge, QVIRL is the first variational approach to Bayesian IRL that explicitly models a joint posterior over optimal Q-values while scaling beyond small tabular environments. The algorithm performs competitively not only against Bayesian IRL baselines, but also against state-of-the-art imitation learning methods. Due to its principled uncertainty quantification of optimal Q-values, QVIRL enables scalable active learning by design. These contributions are a sizable step toward bringing Bayesian IRL methods with calibrated uncertainty into higher-dimensional settings.

\bibliography{zotero_bibtex}
\bibliographystyle{rlj}

\beginSupplementaryMaterials

\section{Method details}
\label{app:method-details}

\subsection{Evaluating Expectations over Next States}
\label{ssec:expectation}
Let us first discuss how to evaluate the expectation over next states in the inverse Bellman equation (Eq.~\ref{eq:bellman-reward} in the main paper) used to derive the reward distribution from the Q posterior:
\begin{equation}
	R(s,a) = Q^*(s,a) - \gamma \,\E_{s'\sim p(\cdot|s,a)}\hat{V}^*(s').
\end{equation}
$\hat{V}^*(s') \approx \max_{a'} Q^*(s',a')$ is the random variable approximating the state-value; its construction is the subject of Section~\ref{app:clark}. For terminal transitions, we set $\gamma=0$, so that $R(s,a)=Q^*(s,a)$ exactly.

\subsubsection{Full Probabilistic Dynamics Model} \label{ssec:prob-qvirl}

In discrete environments, if a probabilistic transition model is available (i.e. we either know the environment dynamics or have some estimated model), we can compute the expectation exactly, by writing the expectation over successor states as a sum weighted by the transition probabilities:
\begin{equation}
    R(s,a) = Q^*(s,a) - \gamma \sum_{s'} p(s'|s,a) \hat{V}^*(s'). 
\end{equation}
This is the version that we use in gridworlds, since we assume their dynamics to be fully known.

In continuous spaces with a transition model, the sum becomes an integral
\begin{equation}
    R(s,a) = Q^*(s,a) - \gamma \int_{s' \in \mathcal{S}} p(s'|s,a) \hat{V}^*(s')\,ds' ,
\end{equation}
which, in general, needs to be approximated numerically. One option is standard Monte Carlo approximation, where we sample a set of possible successor states $s'\sim p(s'|s,a)$ and approximate the integral by their empirical mean. The advantage of this option is that we do not need the full probabilistic model of the environment, but in fact need only access to a simulator that can provide such next-state samples given a state-action pair.

\subsubsection{Empirical Transitions}
In case we cannot sample from a transition model for an arbitrary state-action pair, we need to rely on empirical transitions, i.e. we assume we have access to a set of empirical samples $(s,a,s')$. Such transitions are part of the expert demonstration dataset, but note that the inverse Bellman equation holds for \emph{any} action, not just the expert ones, so such transitions can come from other sources including rollouts of the apprentice policy, a random policy, or any other policy. Such rollouts can start according to the initial state distribution or any other distribution. Such data can either be generated online, during training, or come from a static dataset. The key property we are building on is just the assumption that the data come from the environment of interest, i.e. $s'\sim p(s'|s,a)$.

Whatever the provenance of the data, the expectation then gets crudely approximated as
\begin{equation}
    \E_{s''\sim p(\cdot|s,a)}\hat{V}^*(s'') \approx \hat{V}^*(s')\;,
\end{equation}
which can be interpreted as a single-sample version of the Monte-Carlo approximation from the previous section. Note that this approximation becomes exact in deterministic environments. Moreover, such single-sample estimates are not used in isolation: with enough data, the dataset contains many transitions from proximate states -- or even the same state -- which improves the quality of the approximation for the purposes of the training signal.

\subsection{Clark Approximation to the State-Value Distribution}
\label{app:clark}

This subsection details how we approximate the distribution of $V^*(s)=\max_a Q^*(s,a)$ under the modelling assumption of our method that the optimal Q-values are distributed as a multivariate Gaussian. At its core, this is a standard application of Clark's (\citeyear{clark1961}) moment-matching recursion (see \citealt{hennig2009} for a related moment-matching treatment of correlated-Gaussian maxima in machine learning); in addition, we describe how the covariances of the resulting approximate state-value $\hat V^*(s)$ with the other Gaussian quantities of the model -- the optimal Q-values and the approximate state-values of other states -- follow from the same recursion in closed form. Note that the max-mean approximation $V^*(s)\approx Q^*(s,\argmax_a \mu_{Q^*}(s,a))$ provides a much simpler and easier-to-implement alternative, which still performs fairly well (Appendix~\ref{app:approximations}). When re-implementing this algorithm, we recommend max-mean as a starting point, extending to Clark for slightly higher fidelity once the max-mean version is working.

\paragraph{Moment recursion.}
Fix a state $s$ and let us locally shorten the notation to $\mu_{Q,i}:=\mu_{Q^*}(s,a_i)$ and $\sigma_{Q,i}:=\sigma_{Q^*}(s,a_i)$ for the marginal mean and standard deviation of $Q^*(s,a_i)$; the covariances between the Q-values are those implied by Equation~\ref{eq:q-mean-cov}.

Setting $M_1:=Q^*(s, a_1)$ and $M_i:=\max\bigl(M_{i-1}, Q^*(s,a_i)\bigr)$ for $i=2,\dots,|\as|$, and assuming that $M_{i-1}\approx\mathcal{N}(\mu_{M,i-1},\sigma_{M,i-1}^2)$, Clark approximates the distribution of $M_i$ by $\mathcal{N}(\mu_{M,i},\sigma_{M,i}^2)$ where
\begin{align*}
    \mu_{M,i} &= \mu_{M,i-1}\,\Phi(\nu_i)+\mu_{Q,i}\, \Phi(-\nu_i)+\omega_i \,\phi(\nu_i) \\
    \sigma_{M,i}^2 &= (\mu_{M,i-1}^2 + \sigma_{M,i-1}^2)\, \Phi(\nu_i) + (\mu_{Q,i}^2 + \sigma_{Q,i}^2)\,\Phi(-\nu_i) + (\mu_{M,i-1}+\mu_{Q,i})\, \omega_i\, \phi(\nu_i) - \mu_{M,i}^2
\end{align*}
where
\begin{equation*}
    \omega^2_i = \sigma^2_{M,i-1} + \sigma^2_{Q,i} - 2\, c_i \;, \;\;\;
    \nu_i = (\mu_{M,i-1} - \mu_{Q,i}) / \omega_i ,
\end{equation*}
$\Phi$ and $\phi$ are respectively the CDF and PDF of the standard normal distribution, and $c_i:=\mathrm{Cov}\bigl(M_{i-1},Q^*(s,a_i)\bigr)$ is the covariance between the running maximum and the next Q-value, whose computation we address next. We then approximate the distribution of the final maximum, i.e. the state value $V^*$, by the Gaussian obtained at the last iteration, $\hat V^*(s)\sim \mathcal{N}\bigl(\mu_{V^*}(s),\sigma^2_{V^*}(s)\bigr)$ with $\mu_{V^*}(s):=\mu_{M,|\as|}$ and $\sigma^2_{V^*}(s):=\sigma^2_{M,|\as|}$.

\paragraph{Covariance propagation.}
Because each running maximum $M_{i-1}$ is itself correlated with the remaining Q-values, the covariances $c_i$ cannot be read directly off Equation~\ref{eq:q-mean-cov}; they must be propagated alongside the moments. Clark's method supplies the matching update: for \emph{any} variable $Z$ jointly Gaussian with $Q^*(s,\cdot)$,
\begin{equation}
\label{eq:clark-cov}
    \mathrm{Cov}(M_i, Z) = \Phi(\nu_i)\,\mathrm{Cov}(M_{i-1}, Z) + \Phi(-\nu_i)\,\mathrm{Cov}\bigl(Q^*(s,a_i), Z\bigr),
\end{equation}
initialised by $\mathrm{Cov}(M_1,Z)=\mathrm{Cov}\bigl(Q^*(s,a_1),Z\bigr)$. Taking $Z=Q^*(s,a_i)$ and carrying the recursion to step $i-1$ yields the $c_i$ required by the moment recursion; other choices of $Z$ will yield all the cross-covariances required by our method.

\paragraph{Closed form of the propagated covariances.}
The update in Equation~\ref{eq:clark-cov} is linear in the covariances, so it can be unrolled into a closed form. Collecting the factors $\Phi(\nu_i)$ accumulated across the sweep into the weights
\begin{equation}
\label{eq:clark-weights}
    w_1(s) := \prod_{i=2}^{|\as|}\Phi(\nu_i), \qquad
    w_j(s) := \bigl(1-\Phi(\nu_j)\bigr)\prod_{i=j+1}^{|\as|}\Phi(\nu_i) \;\;\text{ for } j\geq 2,
\end{equation}
which satisfy $w_j(s)\geq 0$ and $\sum_j w_j(s)=1$, the covariance at the final iteration becomes
\begin{equation}
\label{eq:clark-linear}
    \mathrm{Cov}\bigl(\hat V^*(s), Z\bigr) = \sum_{j=1}^{|\as|} w_j(s)\, \mathrm{Cov}\bigl(Q^*(s,a_j), Z\bigr)
\end{equation}
for any $Z$ jointly Gaussian with the Q-values. In other words, for the purpose of computing covariances, $\hat V^*(s)$ behaves like the convex combination $\sum_j w_j(s)\, Q^*(s,a_j)$ of the state's Q-values -- a ``soft argmax'' whose weights reflect how likely each action is to attain the maximum. We stress that this linear surrogate is used only for the covariances: the mean and variance of $\hat V^*(s)$ are the moment-matched values $\mu_{V^*}(s)$ and $\sigma^2_{V^*}(s)$ above, both of which the surrogate would generally underestimate.\footnote{The surrogate mean $\sum_j w_j(s)\,\mu_{Q,j}$ is a convex combination of the action means, so it drops the option-value terms $\omega_i\,\phi(\nu_i)$ and lies below even $\max_j \mu_{Q,j}$; the surrogate variance ignores the variance contributed by the uncertainty about \emph{which} action attains the maximum.}

\paragraph{The induced reward posterior.}
Substituting the Gaussian approximation of the maximum into the inverse Bellman equation (Equation~\ref{eq:bellman-reward}) replaces the intractable nonlinear maximum by the Gaussian $\hat V^*(s')$:
\begin{equation}
\label{eq:bellman-reward-clark}
    R(s,a) = Q^*(s,a) - \gamma \,\E_{s'\sim p(\cdot|s,a)}\!\bigl[\hat V^*(s')\bigr].
\end{equation}
The reward is then an affine functional of jointly Gaussian variables -- the optimal Q-value $Q^*(s,a)$ and the successor values $\hat V^*(s')$ -- so the induced reward posterior is itself Gaussian, with moments given by the following proposition.
\begin{proposition}
\label{prop:reward-gaussian}
Conditional on the Clark approximations $\hat V^*(s')\sim\mathcal N\bigl(\mu_{V^*}(s'),\sigma^2_{V^*}(s')\bigr)$, with covariances given by Equation~\ref{eq:clark-linear}, the reward $R(s,a)$ defined by Equation~\ref{eq:bellman-reward-clark} from the optimal Q-values with mean and covariance from Equation~\ref{eq:q-mean-cov} is normally distributed with mean
\begin{equation}
\label{eq:reward-mean}
    \mu_R(s,a) = \mu_{Q^*}(s,a) - \gamma \,\E_{s'\sim p(\cdot|s,a)}\!\bigl[ \mu_{V^*}(s')\bigr]
\end{equation}
and variance
\begin{equation}
	\label{eq:reward-var-clark}
	\begin{split}
		\scalebox{0.90}{$\sigma^2_R(s,a) = \sigma^2_{Q^*}\bigl(s,a\bigr) - 2 \gamma \, \E_{s'\sim p(\cdot|s,a)}\!\bigl[C_{QV}(s,a;s')\bigr] + \gamma^2 \E_{\substack{s'\sim p(\cdot|s,a)\\ s''\sim p(\cdot|s,a)}}\!\bigl[C_{VV}(s',s'')\bigr]$}
	\end{split}
\end{equation}
where $s'$ and $s''$ are independent draws from $p(\cdot|s,a)$, and the covariances of the successor value with, respectively, the optimal Q-value at $(s,a)$ and the successor value at another sampled state are the bilinear forms
\begin{align}
    C_{QV}(s,a;s') &:= \mathrm{Cov}\bigl(Q^*(s,a),\hat V^*(s')\bigr) = \sum_{j} w_j(s')\;\mathrm{Cov}\bigl(Q^*(s,a),\, Q^*(s',a_j)\bigr), \label{eq:cqv}\\
    C_{VV}(s',s'') &:= \mathrm{Cov}\bigl(\hat V^*(s'),\hat V^*(s'')\bigr) = \sum_{j,l} w_j(s')\, w_l(s'')\;\mathrm{Cov}\bigl(Q^*(s',a_j),\, Q^*(s'',a_l)\bigr) \label{eq:cvv}
\end{align}
for $s'\neq s''$, with the weights of Equation~\ref{eq:clark-weights} and the Q-covariances of Equation~\ref{eq:q-mean-cov}. When $s''=s'$ the two values coincide and $C_{VV}(s',s')=\sigma^2_{V^*}(s')$.
\end{proposition}
\begin{proof}
    We use the fact that if a random vector $X$ in $\R^n$ has a multivariate normal distribution with mean $\mu$ and covariance $\Sigma$, and $A\in\R^{m\times n}$ is a matrix, then $Y=AX$ also has a multivariate normal distribution with mean $A\mu$ and covariance $A\Sigma A^\top$.

    Conditional on the Clark approximations, the successor values $\hat V^*(s')\sim\mathcal N\bigl(\mu_{V^*}(s'),\sigma^2_{V^*}(s')\bigr)$ are jointly Gaussian with the optimal Q-value $Q^*(s,a)$. Collecting them into the vector
    \begin{equation*}
        \mathbf{z} = \Bigl(Q^*(s,a);\; \hat V^*(s')\big|_{s'\in \sts}\Bigr)^\top ,
    \end{equation*}
    Equation~\ref{eq:bellman-reward-clark} can be written as the affine map
    \begin{equation*}
        R(s,a) = A \mathbf{z}\ , \qquad
        A = \Bigl(1, \; -\gamma\, p(s_1|s,a),\; \dots,\; -\gamma\, p(s_n|s,a)\Bigr),
    \end{equation*}
    so $R(s,a)$ is normally distributed with mean $A\mu_{\mathbf z}$ and variance $A\Sigma_{\mathbf z}A^\top$, where $\mu_{\mathbf z}$ and $\Sigma_{\mathbf z}$ are the mean and covariance of $\mathbf z$.

    The mean of $\mathbf z$ is $\mu_{\mathbf z}=\bigl(\mu_{Q^*}(s,a);\,\mu_{V^*}(s')\big|_{s'\in\sts}\bigr)^\top$, and $A\mu_{\mathbf z}$ gives Equation~\ref{eq:reward-mean} directly. The covariance $\Sigma_{\mathbf z}$ has the block structure
    \begin{equation*}
        [\Sigma_{\mathbf z}]_{00} = \sigma^2_{Q^*}(s,a), \quad
        [\Sigma_{\mathbf z}]_{0,s'} = C_{QV}(s,a;s'), \quad
        [\Sigma_{\mathbf z}]_{s',s''} = C_{VV}(s',s''),
    \end{equation*}
    with $C_{QV}$ and $C_{VV}$ as defined in the proposition. Expanding $A\Sigma_{\mathbf z}A^\top$ then yields
    \begin{equation*}
        \sigma^2_R(s,a) = \sigma^2_{Q^*}(s,a) - 2\gamma\sum_{s'} p(s'|s,a)\,C_{QV}(s,a;s') + \gamma^2\sum_{s'}\sum_{s''} p(s'|s,a)\,p(s''|s,a)\,C_{VV}(s',s''),
    \end{equation*}
    which is Equation~\ref{eq:reward-var-clark} written with the successor expectations expressed as sums over the (known) dynamics; the sampled-dynamics cases of Section~\ref{ssec:expectation} replace these sums by their Monte-Carlo or single-sample estimates.

    It remains to justify the covariance entries. Applying Equation~\ref{eq:clark-linear} at $s'$ with $Z=Q^*(s,a)$ gives Equation~\ref{eq:cqv}. Applying it at $s'$ with $Z=\hat V^*(s'')$, and then once more at $s''$ with $Z=Q^*(s',a_j)$ for each $j$, gives Equation~\ref{eq:cvv}. (Equivalently, both can be obtained by carrying the second variable as an external term through the covariance recursion of Equation~\ref{eq:clark-cov}; the linearity of the recursion makes the two computations identical.) When $s''=s'$, the exact Clark variance $\sigma^2_{V^*}(s')$ from the moment recursion is used instead of the surrogate value $\sum_{j,l} w_j(s')\,w_l(s')\,\mathrm{Cov}\bigl(Q^*(s',a_j),Q^*(s',a_l)\bigr)$, which would underestimate it.
\end{proof}

The proposition supplies the marginal reward moments used by the per-transition KL terms. When the KL is instead evaluated \emph{jointly} over a batch of transitions (Appendix~\ref{app:kl}), we also need the covariances between the rewards at different state-action pairs. These follow from the same affine-map argument; we state the result for the empirical-transition setting of Section~\ref{ssec:expectation}, in which the joint KL is used and each pair comes with a single sampled successor state.

\begin{corollary}[Joint reward posterior]
\label{cor:joint-reward}
Let $\{(s_i,a_i)\}_{i=1}^m$ be state-action pairs with respective successor states $\{s_i'\}_{i=1}^m$ and per-pair discounts $\gamma_i$ (where $\gamma_i=\gamma$, except $\gamma_i=0$ for a terminal transition, making $R_i=Q^*(s_i,a_i)$ exactly). Conditional on the Clark approximations, the rewards $R_i := Q^*(s_i,a_i)-\gamma_i \hat V^*(s_i')$ are jointly Gaussian with means $\mu_R(s_i,a_i)$ given by Equation~\ref{eq:reward-mean} and covariance matrix
\begin{equation}
\label{eq:joint-reward-cov}
    [\Sigma_R]_{ij} = \mathrm{Cov}\bigl(Q^*(s_i,a_i),Q^*(s_j,a_j)\bigr)
    - \gamma_j\, C_{QV}(s_i,a_i;s_j')
    - \gamma_i\, C_{QV}(s_j,a_j;s_i')
    + \gamma_i\gamma_j\, C_{VV}(s_i',s_j'),
\end{equation}
where for $i\neq j$ the term $C_{VV}$ is evaluated by Equation~\ref{eq:cvv} and on the diagonal by the exact Clark variance, $C_{VV}(s_i',s_i')=\sigma^2_{V^*}(s_i')$, so that $[\Sigma_R]_{ii}$ recovers the marginal variance of Proposition~\ref{prop:reward-gaussian}.
\end{corollary}
\begin{proof}
    Identical to the proposition: stack $\mathbf z=\bigl(Q^*(s_1,a_1),\dots,Q^*(s_m,a_m);\,\hat V^*(s_1'),\dots,\hat V^*(s_m')\bigr)^\top$ and apply the affine map $A=\bigl(I_m,\ -\mathrm{diag}(\gamma_1,\dots,\gamma_m)\bigr)$.
\end{proof}

\paragraph{Positive semi-definiteness.}
If every entry of Equation~\ref{eq:joint-reward-cov}, including the diagonal, is evaluated through the linear surrogate, then $\Sigma_R = B\,\Sigma_{\bq}B^\top$, where $\Sigma_{\bq}$ is the joint Gaussian covariance of all the involved Q-values and the $i$-th row of the weight matrix $B$ assigns weight $1$ to $Q^*(s_i,a_i)$ and $-\gamma_i\, w_j(s_i')$ to each $Q^*(s_i',a_j)$; $\Sigma_R$ is therefore positive semi-definite by construction. Lifting the diagonal to the exact Clark variances, as in Corollary~\ref{cor:joint-reward}, amounts to adding the residuals $\gamma_i^2\bigl(\sigma^2_{V^*}(s_i') - \sum_{j,l} w_j(s_i')\,w_l(s_i')\,\mathrm{Cov}\bigl(Q^*(s_i',a_j),Q^*(s_i',a_l)\bigr)\bigr)$ to the diagonal, clamped below at zero; since only non-negative values are added to the diagonal, positive semi-definiteness is preserved.\footnote{In practice, we additionally symmetrise $\Sigma_R$ and add a small jitter to its diagonal before the Cholesky factorisation used in the KL computation. A cheaper variant of the joint covariance would skip the diagonal lift and use the purely bilinear (surrogate) form throughout, at the cost of underestimating the marginal variances.} This construction -- propagating maxima of correlated Gaussians in a linear canonical form while moment-matching the marginal distributions -- parallels the canonical delay models used to propagate correlated maxima at scale in statistical static timing analysis \citep{visweswariah2004}.

\subsection{KL and the loss}\label{app:kl}
The loss (Eq.~\ref{eq:elbo}) contains the term $\text{KL}(q_R(R;\theta) || p(R))$, whose purpose is to minimize the distance of the implicit variational posterior over rewards, $q_R(R;\theta)$, to the prior $p(R)$. In tabular environments, we use a multivariate Gaussian prior; in other environments, we assume, more generally, a Gaussian process prior, which can be evaluated on any finite set of state-action pairs to get a multivariate Gaussian prior over those. Similarly, Section~\ref{ssec:implicit-reward} explains how we can obtain a Gaussian posterior over rewards corresponding to the trained variational posterior over Q-values.

Unless the state-action space is very large, the tabular case is trivial: the KL between two multivariate Gaussian distributions with respective means $\mu_1,\mu_2\in\R^d$ and covariance matrices $\Sigma_1,\Sigma_2\in\R^{d\times d}$ can be calculated analytically as
\begin{equation}
\label{eq:kl-gaussians}
    \text{KL}\bigl(\mathcal{N}(\mu_1,\Sigma_1)\,\|\,\mathcal{N}(\mu_2,\Sigma_2)\bigr)
    = \frac{1}{2}\left[
        \operatorname{tr}\!\bigl(\Sigma_2^{-1}\Sigma_1\bigr)
        + (\mu_2-\mu_1)^\top \Sigma_2^{-1} (\mu_2-\mu_1)
        - d
        + \ln\frac{\det\Sigma_2}{\det\Sigma_1}
    \right],
\end{equation}
where $d$ is the dimension, $\operatorname{tr}(\cdot)$ the trace, and $\det(\cdot)$ the determinant. Taking $\mathcal{N}(\mu_1,\Sigma_1)=q_R(R;\theta)$ and $\mathcal{N}(\mu_2,\Sigma_2)=p(R)$, applied to the full joint prior and posterior, gives the required KL term of the loss function. Note that in this tabular case, the KL is evaluated across all state-action pairs, not just the expert trajectories. The ELBO-like loss thus preserves its natural balancing of the likelihood and KL inherited from Bayes' theorem. 

In continuous spaces, we use a Gaussian process prior, and the variational posterior over Q-values implicitly induces a Gaussian posterior over rewards, whose joint mean and covariance over any batch of transitions are given by Corollary~\ref{cor:joint-reward}. The KL term should again encourage closeness between the reward prior and posterior everywhere, not just on the expert trajectories. We use a solution that is standard in VI: together with each mini-batch $B_{\text{expert}}$ of expert data, we also sample a batch $B_{\text{aux}}$ of auxiliary transitions -- these can be sampled from rollouts of the current apprentice policy, an offline dataset, or just be sampled randomly from the state-action space, if we have the requisite level of access. We then calculate the KL from Eq.~\ref{eq:kl-gaussians} jointly over both the expert data and auxiliary data, i.e. over the joint batch $B_{\text{all}}=B_{\text{expert}} \cup B_{\text{aux}}$. We normalize per-point and weight this by a coefficient $\lambda$ to combine this with the likelihood in the continuous, mini-batched version of the loss
\begin{equation}
    \mathcal{L}_{\text{cont}}\biggl(B_{\text{expert}}, B_{\text{aux}}\biggr) = \frac{\lambda}{|B_{\text{all}}|}KL\biggl(q_R(R_{B_{\text{all}}};\theta)\bigg|\bigg|p(R_{B_{\text{all}}})\biggr) - \sum_{s,a\in B_{\text{expert}}} \log p(a|s;\theta)\;,
\end{equation}
where $R_{B_{\text{all}}}$ is a random vector representing the unknown reward at the state-action pairs in $B_{\text{all}}$, with the multivariate Gaussian $q_R(R_{B_{\text{all}}};\theta)$ given by Corollary~\ref{cor:joint-reward}.

Relative to the tabular case, this unfortunately introduces the hyperparameter $\lambda$, since we no longer get the natural weighting of the likelihood and KL implied by the Bayes' theorem. Using trainable inducing points is an alternative that can preserve the weighing and may be worth exploring further, but constructing sufficiently dense inducing points in higher-dimensional spaces is challenging and can introduce other approximation errors.

\subsection{State-only reward variant}
Some environments use rewards that are function of only the state, not the full state-action pair. In theory, QVIRL can model this in its default state-action-dependent model using strong correlations between rewards of all actions in each state to approximate the actual perfect correlation arbitrarily closely. However, this can cause numerical issues -- in the tabular case, QVIRL would be modeling uncertainty over the $|\sts|$-dimensional reward space using a distribution over $|\sts|\times|\as|$ Q-values, causing the joint posterior to be singular. However, this can easily be resolved by instead working in terms of optimal state values $V^*$. 

In the tabular case, the variational posterior can be defined in terms of the means of those state values and their covariance matrix. The rewards can then be derived as
\begin{equation*}
    R(s) = V^*(s) - \gamma \max_a \E_{s'|s,a} V^*(s')
\end{equation*}
and their distribution then approximated as Gaussian using either of the two proposed approximations to the maximum of jointly Gaussian random variables: Clark or max-mean.
The Q-values can then be derived using the linear equation
\begin{equation*}
    Q^*(s,a) = R(s) + \gamma \E_{s'|s,a}V^*(s')
\end{equation*}
yielding again an approximate posterior Gaussian distribution, which can be used to evaluate the likelihood using Lu's approximation (Eq.~\ref{eq:mean-field-action-prob}).
\section{Approximations: Further Notes and Illustration}
\label{app:approximations}

\subsection{Max-Mean Approximation}

\begin{figure}
    \centering 
    \includegraphics[width=0.6\textwidth]{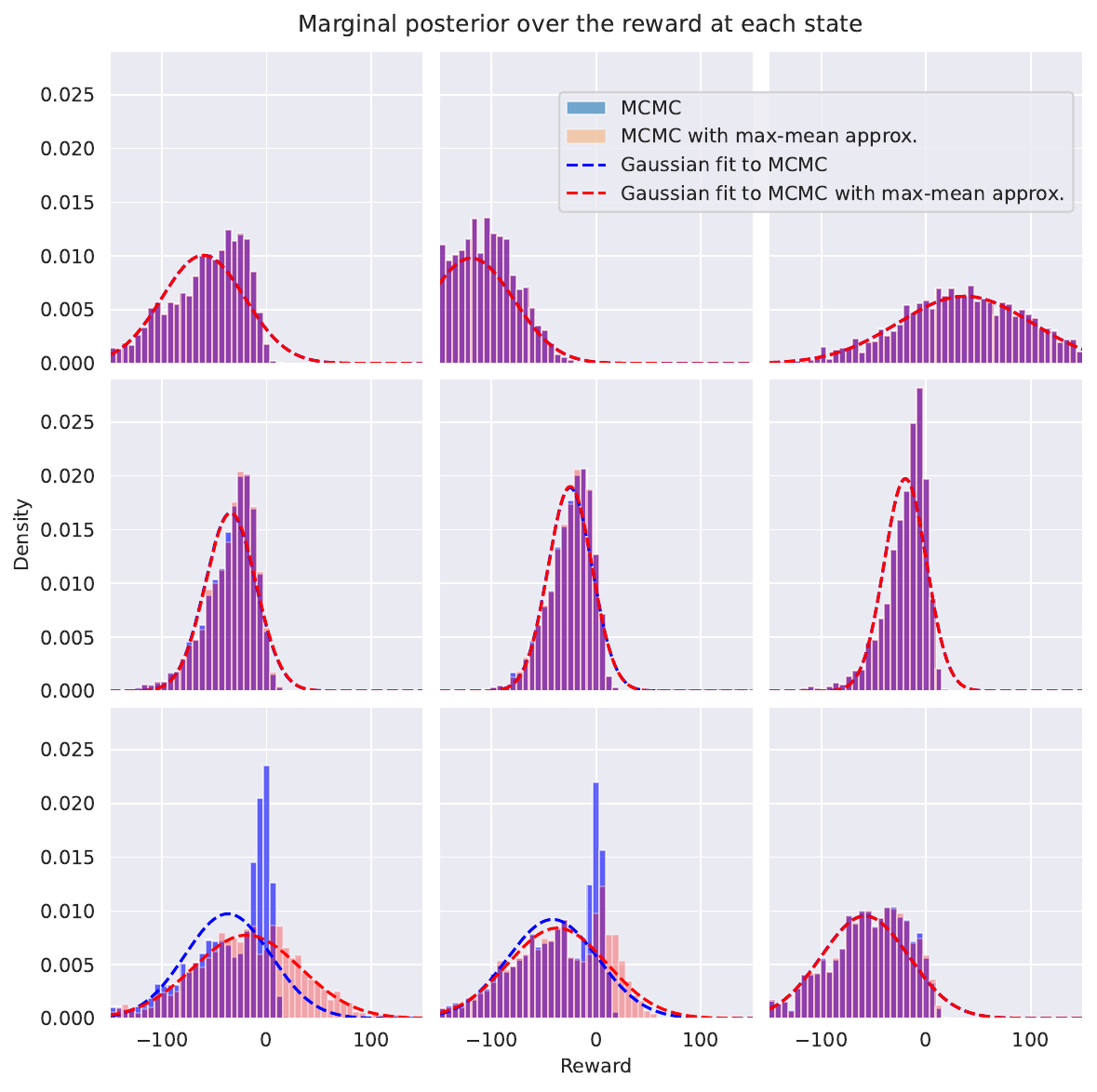} 
    \caption{Effect of the approximation in Eq.~\ref{eq:bellman-reward} on the reward posterior. The plot shows the MCMC posterior distribution over rewards for each state, as well as samples deduced from the MCMC posterior over rewards deduced using the approximation from the posterior over state-values. Dashed lines show Gaussians fitted to the two MCMC samples. The results coincide except in the bottom left and middle squares.}. 
    \label{fig:maxmean}
\end{figure}

In \eqref{eq:bellman-reward}, when approximating the reward distribution corresponding to the current posterior over optimal Q-values, we suggest, as a simpler alternative to Clark's approximation, replacing the distribution of the state-value of the next state, i.e., the maximum of the optimal Q-values for that state, by a closed-form surrogate in order to benefit from computational speed-ups at the cost of approximation error. Whilst the main paper proposes Clark's approximation, there is an even simpler alternative one could consider: replacing the maximisation operation over successor state optimal Q-values, with the distribution of the optimal Q-value for the action with the highest posterior mean. We call this approximation the \emph{max-mean} approximation.

More formally, in a Bellman backup, we replace
\begin{equation*}
	R(s,a) = Q^*(s,a) - \gamma \,\E_{s'\sim p(\cdot|s,a)}\!\left[\max_{a'} Q^*(s',a')\right].
\end{equation*}
with the following:
\begin{equation}
	R(s,a) = Q^*(s,a) - \gamma \,\E_{s'\sim p(\cdot|s,a)}\!\left[ Q^*\left(s',\argmax_{a'}\mu_Q(s',a')\right)\right].
\end{equation}

Figure~\ref{fig:maxmean} illustrates the effect of the approximation in a 3x3 gridworld. The approximation does not introduce any error in most states. This is because the posterior over optimal Q-values in the successor states to these states clearly favours a single action and thus the distribution of the state value coincides with the distribution of the highest-mean action. On the other hand, we can see the approximation introducing an error in two states at the bottom left. This is because the algorithm's uncertainty about what the optimal action in the next state is, so the state-value does not coincide with the distribution of the max-mean optimal Q-value. Consequently, the algorithm slightly underestimates the expected value of the next state and overestimates the reward.

\subsection{Likelihood Approximation}

\begin{figure}
    \centering 
    \includegraphics[width=0.49\textwidth]{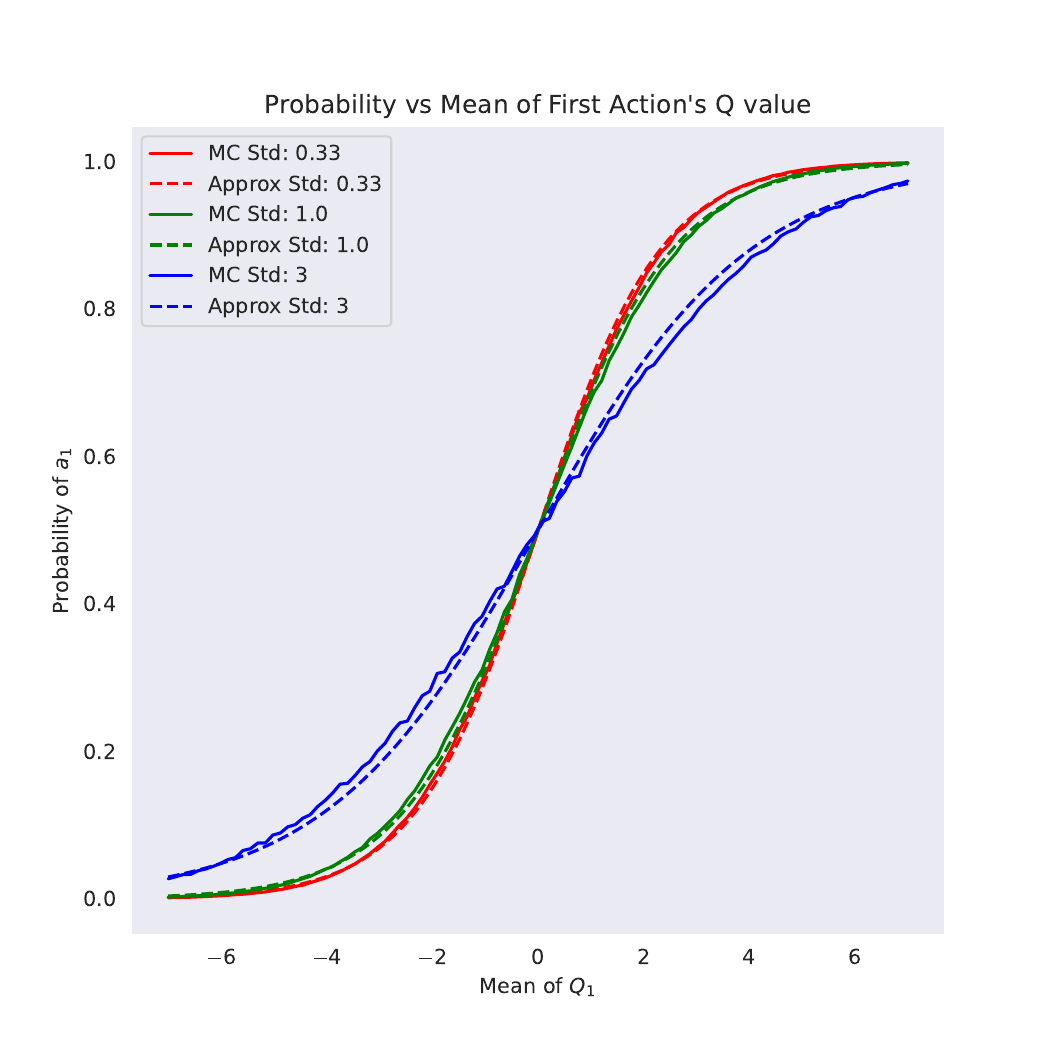}  \includegraphics[width=0.49\textwidth]{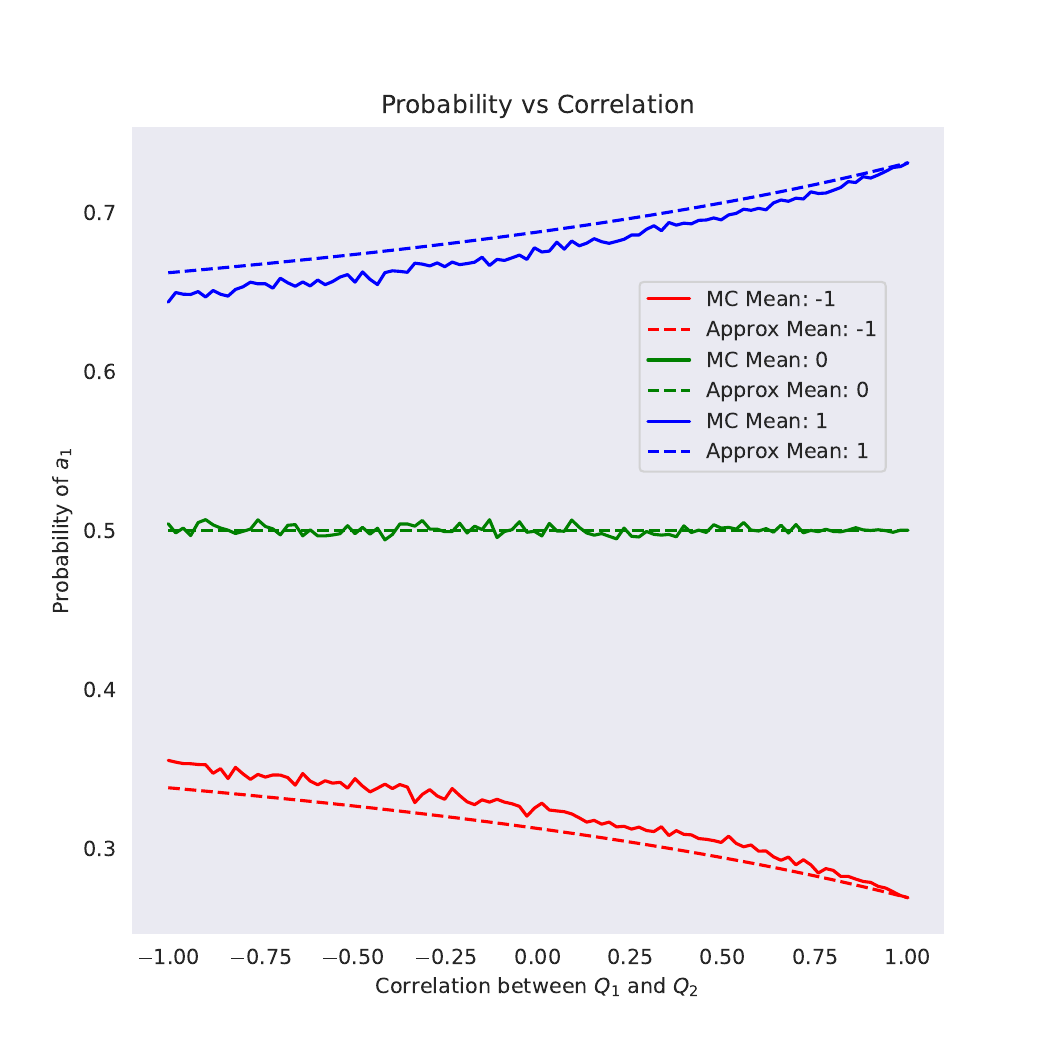}
    \caption{Comparison of the likelihood approximation \eqref{eq:mean-field-action-prob} against a Monte Carlo evaluation using 10,000 Monte Carlo samples. Left: probability of $a_1$ as a function of the mean of $Q_1$ for 3 levels of standard deviation of $Q_1$. Correlation is set to 0 (the plot looks very similar for correlated values). Right: Probability of $a_1$ as a function of the correlation between $Q_1$ and $Q_2$ for three levels of the mean. Standard deviation is fixed to 1. Note the different scales on y-axes of the two plots.} 
    \label{fig:likelihood-approximation}
\end{figure}

Following~\citet{lu2021}, in Equation~\ref{eq:mean-field-action-prob}, we approximate the likelihood 
\begin{equation*}
    p(a|s;\theta) = \int_{Q(s,\cdot)\in\R^{|\as|}} \frac{e^{\beta Q(s,a)}}{\sum_{a'}e^{\beta Q(s,a')}} q_\theta(Q(s,\cdot)) \; d Q(s,\cdot)
\end{equation*} 
as
\begin{equation*}
    p(a|s;\theta)\approx \scalebox{0.925}{$\left(\sum_{a'} \exp\left(- \frac{\beta\left(\mu_Q(s,a) - \mu_Q(s,a')\right)}{\sqrt{1+ 3\pi^{-2} \beta^2 (\sigma_Q(s,a)^2+\sigma_Q(s,a')^2-2\Sigma_{aa'}) }}\right) \right)^{-1}$}.
\end{equation*}

To examine the quality of the approximation, we compare the probability of the action under the approximation against an approximation of the integral using Monte Carlo integration with 10,000 samples. In the test scenario, there are two actions, $a_1$ and $a_2$ with respective Q-values $Q_1$ and $Q_2$ that are jointly-normally distributed. We fix the marginal distribution of $Q_2$ to standard normal. We then try varying the mean and standard deviation of $Q_1$ as well as the correlation between $Q_1$ and $Q_2$.

Figure~\ref{fig:likelihood-approximation} shows the results. The approximation seems reasonably good, though it gets worse as correlation decreases. 

Note that even with 10,000 samples, the Monte Carlo estimate displays notable noise. While we could use sampling and the reparameterization trick during training, we observed that this destabilizes training and we did not manage to get competitive results on any of the benchmarks. Thus we recommend using the proposed approximation that can be evaluated analytically.

\subsection{Gaussianity}

The Gaussianity assumption was illustrated in the previous example. Fig.~\ref{fig:3x3-gridworld-multirew} shows that for a Gaussian prior, the posterior stays reasonably close to Gaussian. The use of variational inference requires an assumption on the distribution family and the Gaussian is a natural starting point, on which further work can build to expand to other distribution families, especially if practical cases arise where the prior and associated posterior are notably non-Gaussian.
\section{Additional Discussion on AVRIL}

Here, we provide additional details about our modifications to AVRIL that we implemented for fair comparison between the methods. Additionally, we give a visualisation of AVRIL's shortcomings in a $3\times3$ gridworld setting, and conclude with a short summary of limitations of AVRIL as a Bayesian IRL method.

\subsection{Dynamics-Aware AVRIL}
\label{app:avril-online}
AVRIL was originally proposed as an offline algorithm not using the environment dynamics or auxiliary transition. However, in online settings, this would lead to AVRIL never updating the posterior over the reward associated with states visited during environment rollouts, as AVRIL only updates reward posteriors of states that appear in the demonstrations. Here we assume the knowledge of dynamics, so for fair comparison we introduce a dynamics-aware version of AVRIL.

The main change is that instead of evaluating both the TD term and the KL divergence between the reward variational posterior and the prior only on the expert trajectories, we evaluate it across all state-action pairs, i.e. for each state-action pair $(s,a)$, we calculate
\begin{equation*}
    R(s,a) = Q^*_\theta(s,a) - \gamma \E_{s'|s,a} V^*_\theta(s'),
\end{equation*}
where $V^*_\theta(s')=\max_{a'} Q^*_\theta(s',a')$. Then the TD soft constraint is calculated as
\begin{equation}
    \sum_{s,a\in\sts\times\as} q_\phi(R(s,a)).
\end{equation} Similarly, we evaluate the KL divergence between $q_\phi$ and the prior $p(R)$ across all states and actions rather than just on the expert trajectories.

\subsection{3x3 Gridworld and a Comparison to AVRIL}
\label{app:avril-gw}

\begin{figure}
    \centering 
    \includegraphics[width=0.4\textwidth]{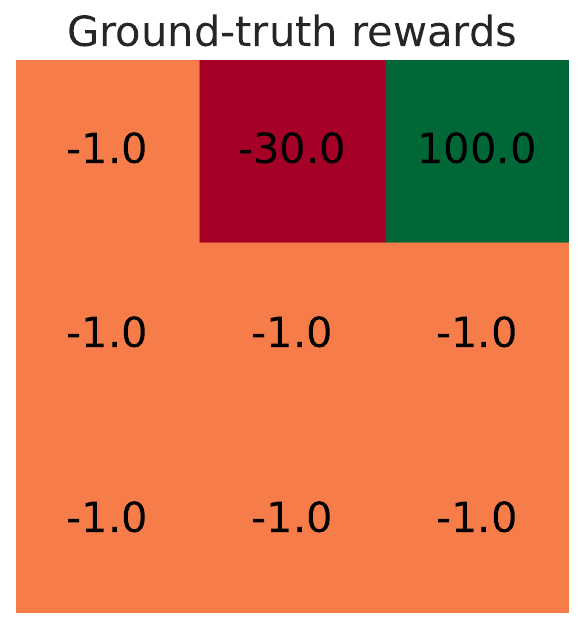} \includegraphics[width=0.4\textwidth]{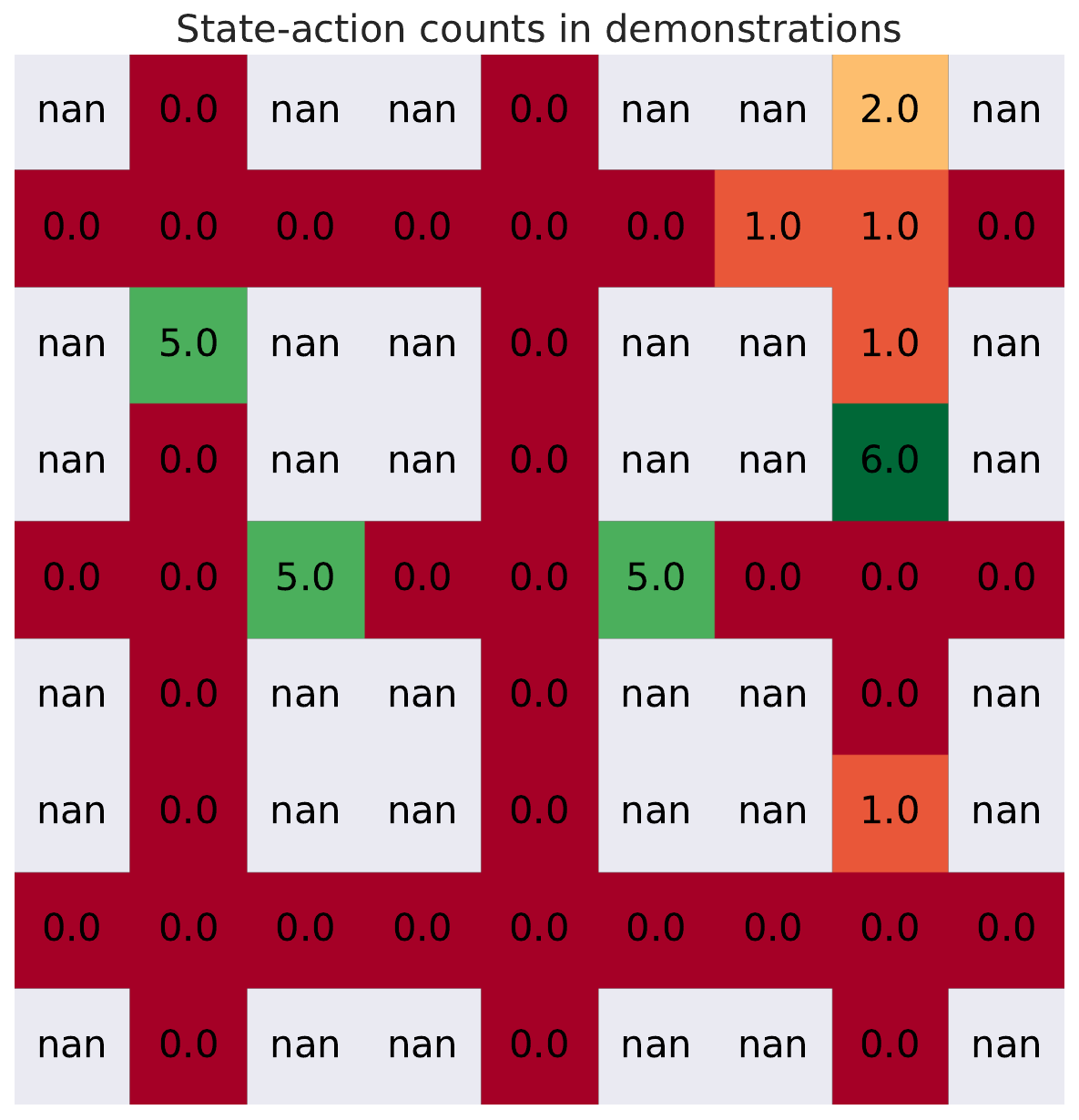} \\\includegraphics[width=0.48\textwidth]{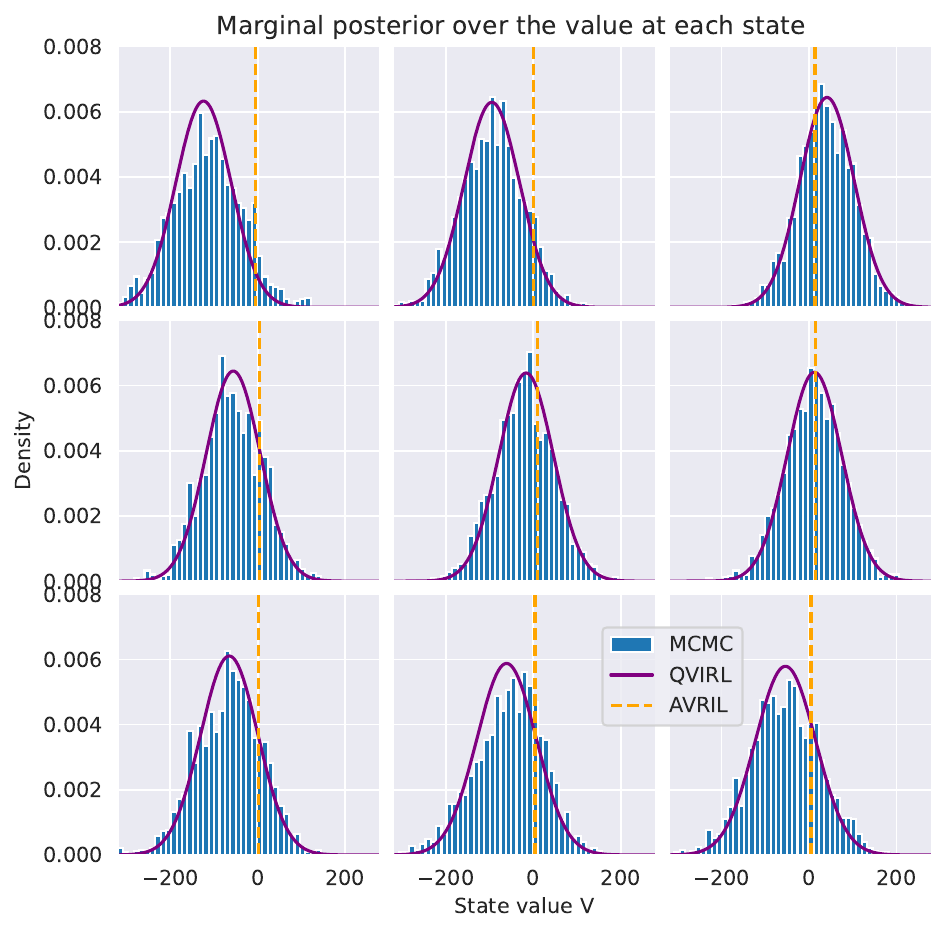} \; \includegraphics[width=0.48\textwidth]{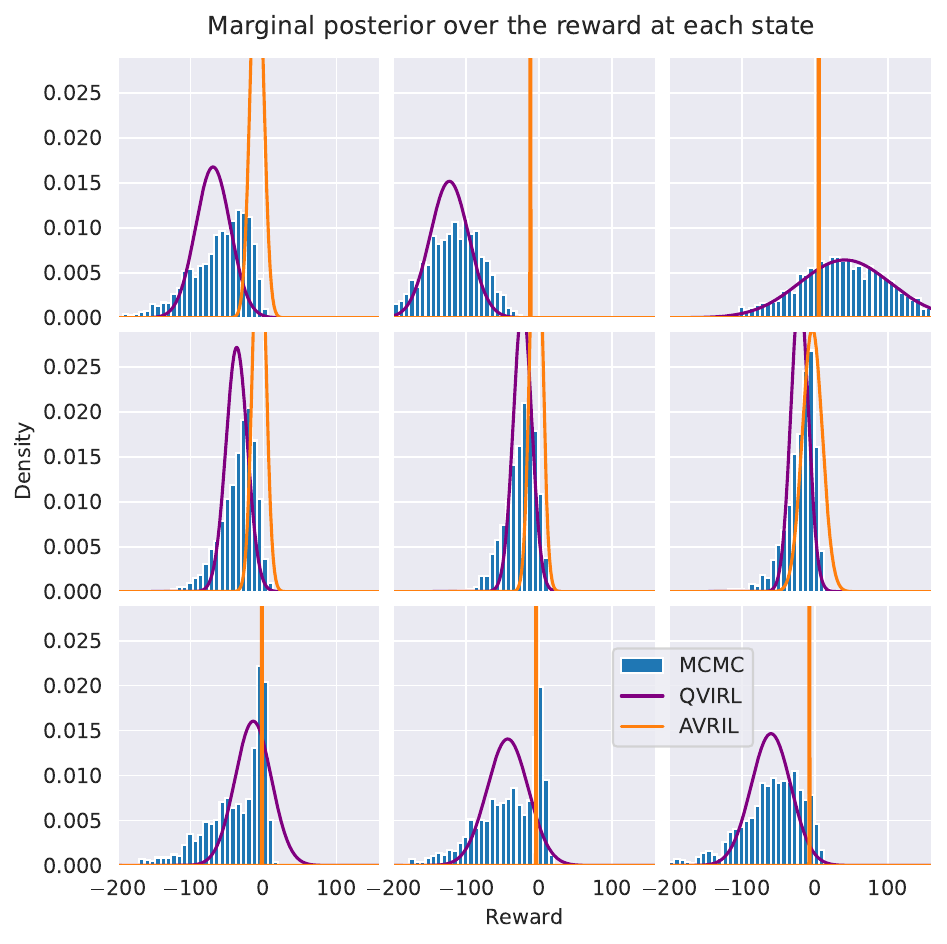}
    \caption{Top left: Ground truth rewards for a 3x3 gridworld. The top left state is the initial state. Top right state is terminal. Top right: State-action counts in the demonstration data. Bottom left: pdf of the marginal posterior over state values recovered by QVIRL, a histogram of 2000 samples from the true posterior produced by ValueWalk, and a horizontal line indicating the point estimate recovered by AVRIL. Bottom right: pdfs of the variational posterior over reward recovered by QVIRL and AVRIL and the corresponding histogram.}
    \label{fig:3x3-gridworld-multirew}
\end{figure}

To visualise the shortcomings of AVRIL when it comes to recovering true posteriors of the state values, we conducted an experiment in a simple $3\times3$ gridworld shown in Figure~\ref{fig:3x3-gridworld-multirew}. We treated all state-only rewards as unknown. We assume an independent normal prior $\mathcal{N}(0,66)$ for the reward in every state. There is also noise in the environment dynamics, so for any action, there is a $0.1$ probability of slipping and moving in a random direction or staying in place (with uniform probability between the 5 options). The top right state yields a reward of 100, while the middle top tile represents a hazardous obstacle with a reward of -30.

We provide the algorithms with 5 trajectories, each starting in the top left corner and heading to the goal via the middle tile while avoiding the obstacle. The expert once slipped to the bottom right tile placing one demonstration step there.

We ran QVIRL and a dynamics-aware version of AVRIL (described in Appendix~\ref{app:avril-online} above) using these demonstrations, as well as ValueWalk~\citep{bajgar2024} that produces samples from the true posterior.

Figure~\ref{fig:3x3-gridworld-multirew} shows that QVIRL variational posterior approximates the true posterior well (as is also the case in Section~\ref{sec:experiements}), while the AVRIL point estimate remains very close to zero (which, however, still produces action predictions consistent with the data). The fit of the derived posterior over the rewards is worse -- some of the posteriors are visibly further from Gaussian (many exhibiting negative skew) than their value counterparts. Still, the derived Gaussian posterior over rewards recovered by QVIRL still roughly matches the means and variances of the MCMC samples. On the other hand, the posterior over rewards recovered by AVRIL remains centred near zero in almost all cases. Also, in 5 of the 9 states, the standard deviation of the posterior collapses to below 0.2, where it should be between 40 and 64 to match the MCMC-based posterior. 

Furthermore, as discussed above, we found AVRIL to be extremely sensitive to the value of its hyperparameter $\lambda$ that regulates the strength of the constraint on consistency between the Q-values and the reward posterior. In the above experiments we used a value of $\lambda=0.2$, however, increasing it to $0.5$ makes the variance of all states collapse to near zero (std<0.001 for 8 states out of 9), while decreasing to $0.1$ essentially reverts the posterior to the prior. The AVRIL paper gives no indication of how to pick a good value for $\lambda$.

\subsection{Limitations of AVRIL}

Since QVIRL may seem to resemble AVRIL, and AVRIL can be considered as the closest baseline, we would like to re-emphasize some important limitations of AVRIL that our method does not suffer from:
\begin{itemize}
    \item AVRIL does not provide uncertainty estimates over the optimal Q-values and thus over the optimal policy. This does not permit one to easily extract risk averse policies. On the other hand, doing so is natural in QVIRL, which gives us a downstream use gain for comparable computational effort.
    \item AVRIL's posterior over rewards does not seem to track well the true posterior over rewards even on a simple gridworld. QVIRL's fit is faithful to the posterior recovered by MCMC methods.
    \item the posterior variance over rewards is extremely sensitive to AVRIL's hyperparameter $\lambda$ that regulates the strength of the consistence constraint between the reward variational distribution and the Q-values and a narrow range of values can make the posterior either collapse to Dirac delta-like function, or revert to the prior. The paper gives no indication on how to pick the value. QVIRL has no such hyperparameter, and is robust even to other potentially challenging design choices such as the weighting of KL divergence and log-likelihood in ELBO.
    \item the AVRIL paper presents a "strictly batch setting" where only the expert demonstrations are used to estimate environment transitions, and thus evaluate the closeness of the posterior to the prior and the consistency between the Q-values and the reward distribution. We think this largely gives up a crucial advantage of IRL over behavioural cloning -- leveraging the environment dynamics to make inference about rewards away from the expert trajectories, enabling better generalization. Our paper presents variants of the algorithm also for the case of known environment dynamics, or an online setting where we have access to a simulator or the environment itself.
\end{itemize}
\section{Experiment Details}
\label{app:exp-details}

\subsection{Environment Descriptions}

\begin{figure}[htpb]
\subfloat[Gridworld environment.]{\label{fig:env-gw}
\centering
\includegraphics[width=0.4\linewidth]{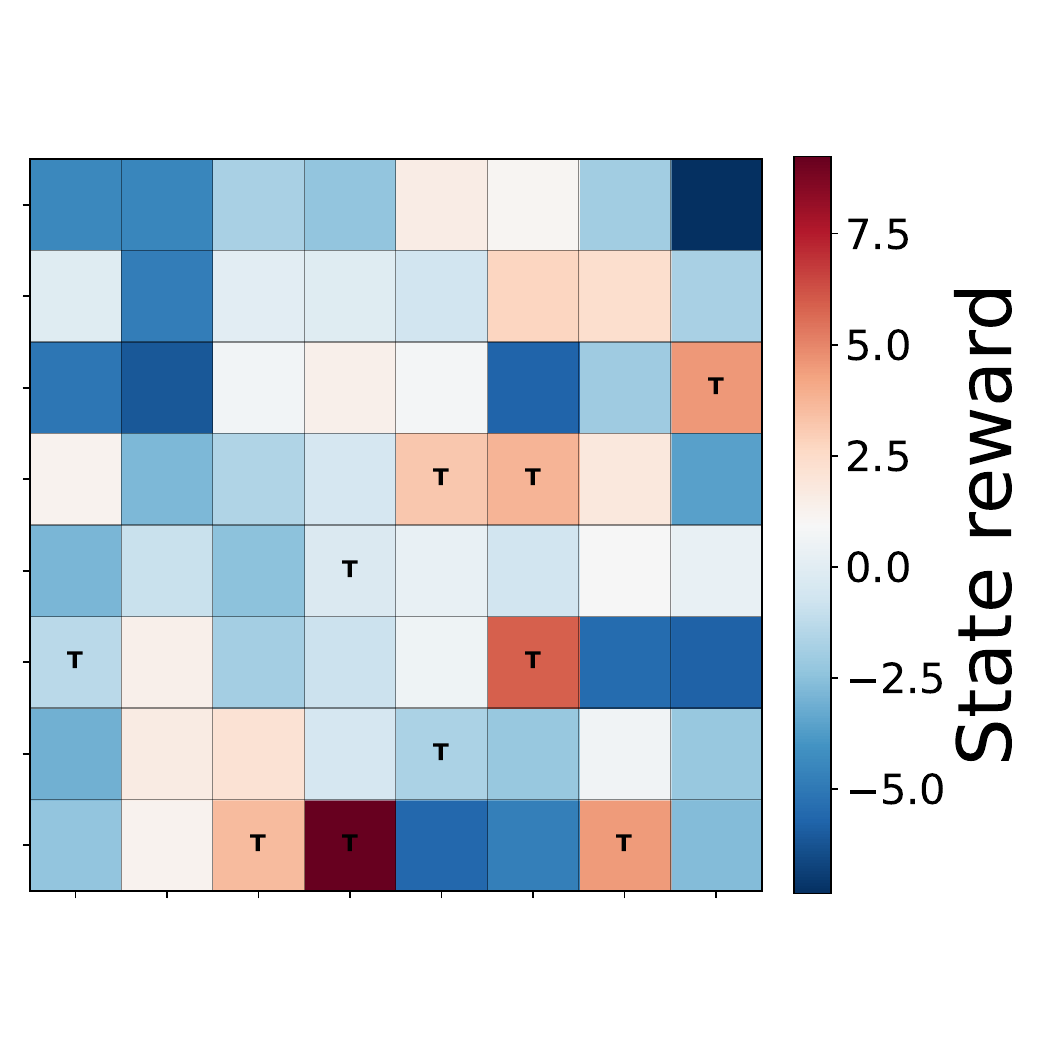}
}
\hfill
\subfloat[Atari Pong environment.]{\label{fig:env-pong}
\centering
\includegraphics[width=0.27\linewidth]{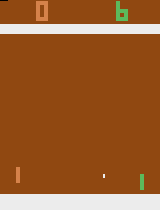}
}
\hfill
\subfloat[Atari Space Invaders.]{\label{fig:env-si}
\centering
\includegraphics[width=0.27\linewidth]{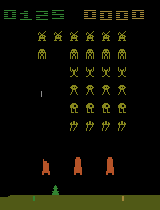}
}

\vfill

\subfloat[Lunar Lander environment.]{\label{fig:env-ll}
\centering

\includegraphics[width=0.45\linewidth]{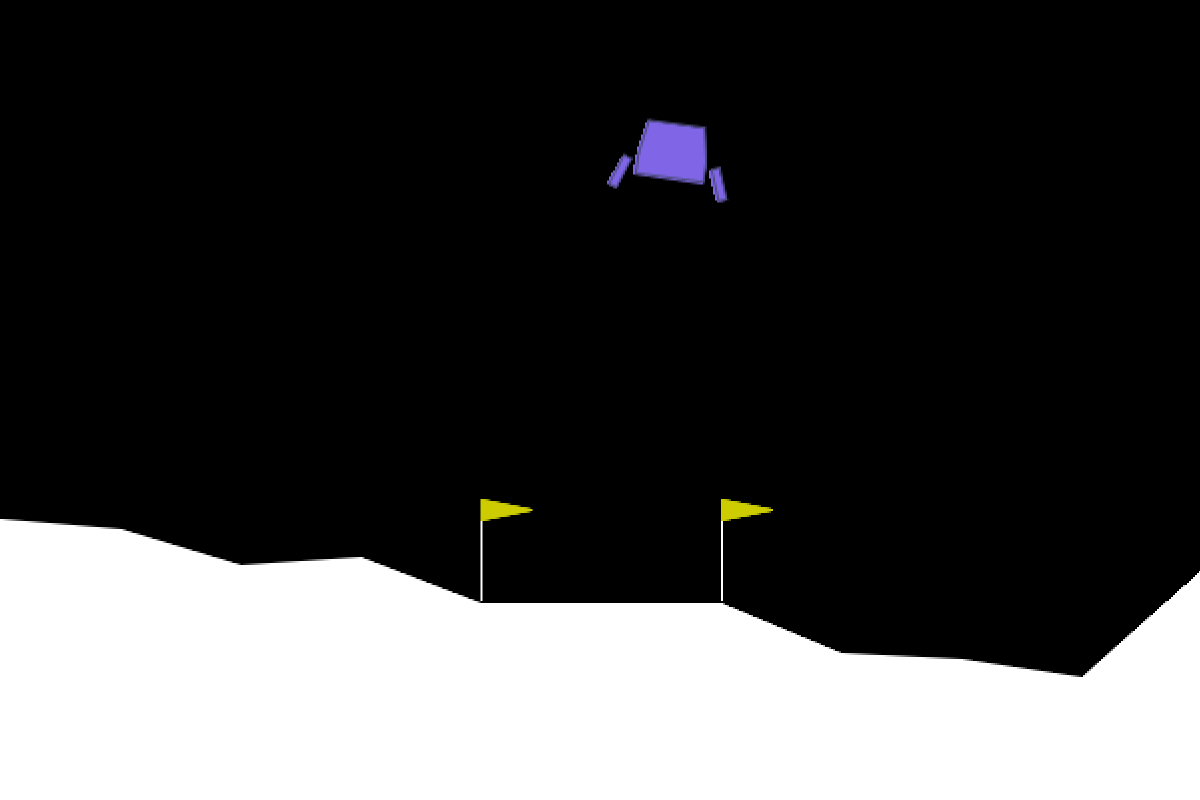}
}
\hfill
\subfloat[Highway Env environment.]{\label{fig:env-he}
\centering

\includegraphics[width=0.45\linewidth]{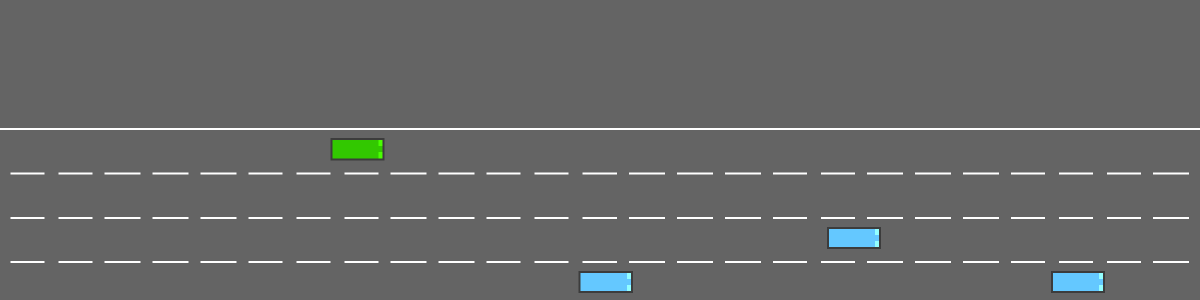}
}

\end{figure}

\paragraph{Gridworld} The gridworld  (Figure~\ref{fig:env-gw}) has 5 actions, corresponding to staying in place and moving in the four directions. Furthermore, there is a probability of 0.1 of random action being executed instead of the intended one. If an action would result in crossing the edge, the agent instead remains in place. The gridworlds use a state-only reward (awarded upon executing any action in the given state). The 100 randomized 8x8 gridworlds were generated as follows:
\begin{enumerate}
    \item Each state was assigned a random reward drawn independently from $\mathcal{N}(-1,3)$ (which can be interpreted as $\mathcal{N}(0,3)$ minus a time penalty of one per step). 
    \item Each state was then marked as terminal with an independent probability of 0.1.
    \item The top $10\%$ of states with highest reward were further marked as terminal (producing terminal goal states, which may, however, sometimes be avoided by the optimal policy in favour of staying forever in other positive states).
    \item The initial state distribution is uniform across the whole state space.
\end{enumerate}

\paragraph{Lunar Lander} Lunar Lander, part of the classic control environments available in Farama foundation's Gymnasium library \citep{towers2024gymnasium}, is a standard benchmark task in reinforcement learning. It models a spacecraft attempting to land on a designated pad located between two flags on a flat surface. The state is an 8-dimensional continuous vector that includes the lander's horizontal and vertical position and velocity, its angle and angular velocity, and indicators of whether the left and right leg are in contact with the ground. We use the discrete variant of the environment, containing 4 actions - a no-op, and control actions of the left and right main and orientation engines. The agent receives positive reward for successfully approaching the landing pad and slowing down. It receives negative reward for excessive tilt and firing of the engines. If the agent successfully lands, it receives an additional reward of $+100$, and if it crashes, it receives a reward of $-100$. An episode is considered solved at an episode return of $+200$, with return of approx. $270$ commonly being considered as expert performance.

For active learning, we implement a version of the environment in which the starting state is resettable from a state vector in order to allow active querying of the expert and targeted rollouts.

\paragraph{Highway Env} Highway Env is a simulated autonomous driving task, also available in Gymnasium, in which an agent has to negotiate a multi-lane highway that is shared with other vehicles. We use the fast version of the environment with kinematic observations and discrete action space. The observed state is a $5 \times 5$ feature vector that contains the 2D position of the ego vehicle (the agent) and its' velocities, as well as the same features for 4 vehicles that are the closest to the ego vehicle in a given state. The actions available to the ego vehicle are a set of discrete meta-actions: change lane (left or right), speed up, slow down, and no-op. The episodes are time-limited to 30 seconds, and the agent receives a reward that is proportional to the difference between its speed and the environment's desired speed. The agent also receives positive reward for keeping in the right-most lane. The agent receives negative reward if it crashes into any other vehicle, or drives off the road. The maximum reward the ego vehicle can obtain is $+30$.

For active learning, we implement a version of the environment in which the starting state is resettable from a true state vector, and we collect demonstrations in the true state space. When training experts and IRL algorithms, we convert the true state space into observations as given by the original environment.

\paragraph{Atari} To test QVIRL's scalability to raw, high-dimensional pixel observations, we evaluate apprenticeship learning on two games from the Arcade Learning Environment \citep{mnih2015}, Pong and Space Invaders, following the evaluation protocol used by IQ-Learn \citep{garg2021}. Observations are pre-processed using the standard Atari pipeline: frames are converted to grayscale and resized to $84\times84$, four frames are skipped between decisions (with max-pooling over the last two skipped frames), and the last four frames are stacked to form the $4\times84\times84$ observation given to the agent, so that short-term dynamics (e.g.\ ball velocity in Pong) are observable from the state alone. The action space is discrete and game-specific. Episode returns are reported on the raw (unclipped) game score, while the reward used for training is sign-clipped as is standard practice for these games. The prior distribution used for the reward is described in Appendix~\ref{app:prior-distribution}.

\subsection{Expert Demonstrations}

To collect expert demonstrations in Gridworlds, we solve for the optimal Q-values using value iteration, and then roll out 5 expert trajectories with a Boltzmann rationality coefficient $\beta=2$ and maximum length 5. In the active learning setting, we start with a single such trajectory and then collect an additional one from the expert in the state queried in each active step.

To collect expert demonstrations in Lunar Lander and Highway Env, we train an expert using QR-DQN \citep{dabney2018}. When the expert's Q-value function has converged, we use it to sample trajectories with different rationality coefficients; $\beta=3$ on Lunar Lander and $\beta=5$ in the highway environment. For each $\beta$, we collect $1000$ trajectories, and randomly split them into $800$ training trajectories, $100$ evaluation trajectories, and $100$ test trajectories. We further split the $800$ training trajectories into $10$ subsets, and we use one training subset during a single method evaluation, or use all $10$ splits across $10$ runs of the same method. Furthermore, we sometimes further restrict the number of trajectories we give to the methods, and in this case we subsample the training splits to create smaller training datasets.

To collect expert demonstrations on Atari, we roll out a Boltzmann-rational policy from a near-optimal pretrained action-value network (an Implicit Quantile Network~\citep{dabney2018a}) for each game, matching IQ-Learn's demonstration budget of 20 training trajectories per game. A further 5 trajectories are held out per game to evaluate action log-likelihood, as in the other environments.

\subsection{QVIRL Training}

In our experiments, we train QVIRL by stochastic gradient ascent using automated differentiation in PyTorch with the Adam optimizer and a learning rate of $0.001$ (using the default values for other parameters). We use the same Boltzmann coefficient as was used to generate the demonstrations, and a discount rate $\gamma=0.99$ for Lunar Lander, $0.95$ for Highway, and $0.9$ for gridworlds. Each run was run for about 20000 iterations with a batch size of 64. Where a neural network encoder is used, we use a multi-layer perceptron with two hidden layers of 128 units and an ELU activation function.
Embedding size of 4 is used in our experiments. We also used weight decay of $0.001$.
We tuned only the learning rate and weight decay, which were chosen by trial and error before the whole set of experiments whose results are reported was run. At most 5 tries were made for each hyperparameter.

On Atari, the state input is first passed through a convolutional (NatureCNN) encoder~\citep{mnih2015} -- three convolutional layers followed by a linear layer -- before entering the same downstream architecture used in the other environments; the encoder is trained end-to-end together with the rest of QVIRL's variational parameters, rather than pretrained or frozen.

\subsection{Prior Distribution}
\label{app:prior-distribution}

In the gridworld experiments, we used an independent Gaussian prior $\mathcal{N}(-1, 3)$ for each state.

In Lunar Lander and Highway, we used a Gaussian-process prior with constant mean $-1$ and an automatic relevance determination (ARD) RBF kernel,
\begin{equation}
    k(x,x') = \sigma_0^2 \exp\left(-\tfrac{1}{2}\sum_d \left(\frac{z_d - z'_d}{\ell_d}\right)^2\right),
\end{equation}
evaluated on standardized features $z = (x - \bar{x})/s_x$, where $x$ concatenates the (flattened) state vector with a one-hot encoding of the action, and $\bar x, s_x$ are the empirical per-feature mean and standard deviation computed once from the expert demonstrations. We set the kernel's marginal standard deviation $\sigma_0$ to 50 on Lunar Lander and 1 on Highway, to roughly track the magnitude of rewards on each environment. The per-feature lengthscales $\ell_d$ are initialized in a data-adaptive way, as the root-mean-square of the lag-one differences of the standardized feature $z_d$ along expert trajectories (differences are never taken across trajectory boundaries). Both the lengthscales and $\sigma_0$ are then treated as learnable parameters, and are optimized jointly with the rest of QVIRL's variational parameters, by the same Adam optimizer, driven end-to-end by the ELBO -- i.e. we do not perform a separate marginal-likelihood (type-II ML) fit of the kernel hyperparameters.

On Atari, we used a zero-mean Gaussian-process prior whose kernel is a product of an action kernel and a weighted mixture of image kernels evaluated on the (grayscale, $84\times84$) frame-stack observation,
\begin{equation}
    k\big((s,a),(s',a')\big) = \sigma_0^2\, k_A(a,a') \sum_j \alpha_j\, \tilde{k}_j(s,s'),
\end{equation}
with each component kernel normalized so that $\tilde k_j(s,s)=1$ and $k_A(a,a)=1$. The mixture combines: (i) a patch-based RBF kernel comparing $8\times8$ image patches (stride 4) at matched spatial locations; (ii) the same patch RBF kernel applied to the difference between consecutive stacked frames, to capture motion (e.g.\ ball or sprite direction); (iii) a parameter-free multi-resolution (spatial-pyramid) kernel on per-cell mean pixel intensities, which captures coarse scene layout; and (iv) a weak RBF kernel on raw pixel intensities, acting as a regularizing baseline. The action kernel $k_A$ is an RBF kernel over hand-specified, per-action semantic features (one-hot encoding each move direction and whether the agent fires), mixed with a small Kronecker-delta term. The mixture weights and the number of pyramid levels are set per game to reflect each game's visual structure -- for Pong, weights of $0.45/0.45/0.05/0.05$ on the patch/motion/pyramid/pixel kernels with pyramid levels $\{0,1,2\}$; for Space Invaders, weights of $0.25/0.35/0.35/0.05$ with pyramid levels $\{0,1,2,3\}$, reflecting the more spatially distributed enemy layout. The patch and pixel-kernel lengthscales are initialized using a median heuristic (the median pairwise distance between patches sampled from the expert demonstrations). We fix the kernel's marginal standard deviation $\sigma_0$ to $0.5$ to roughly track Atari's (clipped) reward magnitude. Unlike in Lunar Lander and Highway, all Atari kernel hyperparameters -- mixture weights and lengthscales alike -- are kept fixed at their (data-informed) initial values rather than trained jointly with the rest of the model.

\section{Additional Experiments}

Here, we provide some additional experiments which are meant to expose and visualise the properties of QVIRL, and shortcomings of AVRIL, discussed in Appendix~\ref{app:avril-online}.

\subsection{Sensitivity of QVIRL to Rationality Coefficient $\beta$}

Whilst QVIRL has relatively few hyper-parameters that require tuning, most current Bayesian IRL methods require the designers to specify the rationality coefficient $\beta$, which controls the sharpness of the softmax operation in modelling the demonstrator's action selection (Equation~\ref{eq:boltzmann-likelihood}). Naturally, the (mis-)specification of $\beta$ influences the reward learning, as learning signal might be biased or lost due to misaligned assumptions.

To study the influence of misspecification of the rationality coefficient $\beta$ on QVIRL's learning process, we identified 4 values of $\beta$ that produced distinct behaviours in gridworlds, ranging from random actions, to optimal experts. We then trained QVIRL by using all values of $\beta$ for each expert. We repeated the experiment across 30 seeds, and we considered 1, 7, and 15 demonstrations. As our metric of interest, we focus on the discrepancy of the held-out greedy expert action likelihood between the demonstrator and the learner, that is, we measure the difference in probability mass assigned to the true optimal action in a state. We report the mean of this metric across all seeds, and a positive difference means the learner places less mass on the optimal action, and vice versa for a negative difference.

\begin{figure}[t]
    \centering
    \includegraphics[width=0.32\linewidth]{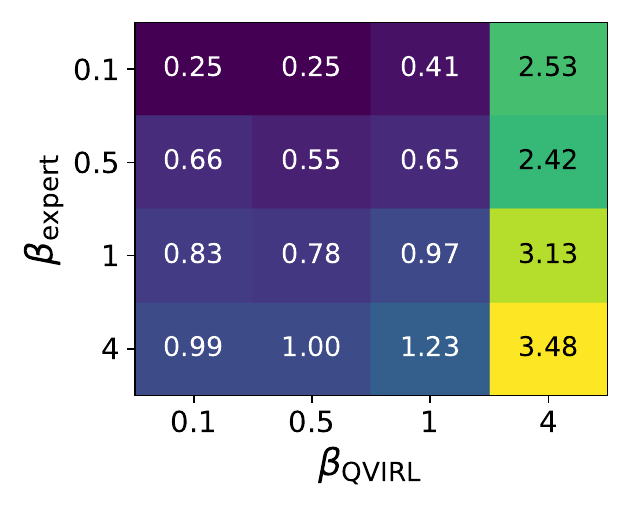}
    \hfill
    \includegraphics[width=0.32\linewidth]{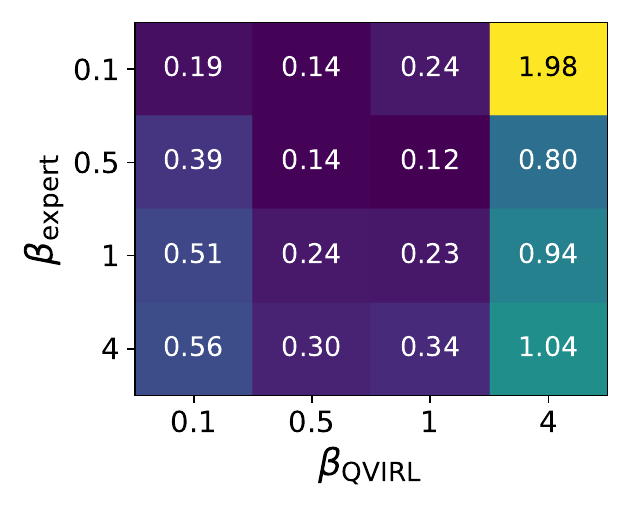}
    \hfill
    \includegraphics[width=0.32\linewidth]{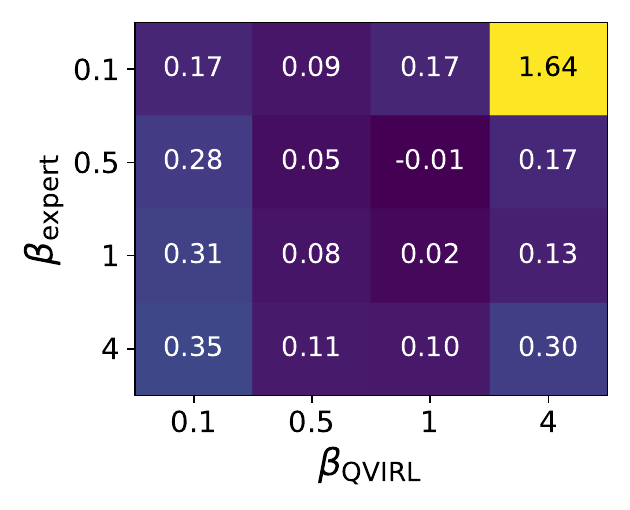}
    \caption{The mean true greedy action log probability difference between the demonstrator and learner for 1, 7, and 15 demonstrations, respectively.}
    \label{fig:diffs}
\end{figure}

As observed in Figure~\ref{fig:diffs}, the interaction between the $\beta$ parameter and learned action probabilities does not only depend on the misspecification, but also on the amount of demonstration data given to QVIRL. In low-data regimes, QVIRL performs better when assuming a less deterministic expert than the true value, and interestingly, best-performance is achieved by deliberately under-estimating the expert's rationality. However, as the number of data increases, its regularisation effect is more pronounced, and the variance is better explained by the rationality coefficient. In terms of absolute values, it appears to be better to slightly underestimate the expert's rationality rather than overestimate it.

An exciting direction for future work is to include the rationality parameter $\beta$ as an explicit part of the inference, along the optimal Q-values, to be estimated from demonstration data.

\subsection{Illustrative Gridworld Experiment with a Single Unknown State}
\begin{figure}
    \centering 
    \includegraphics[width=0.25\textwidth]{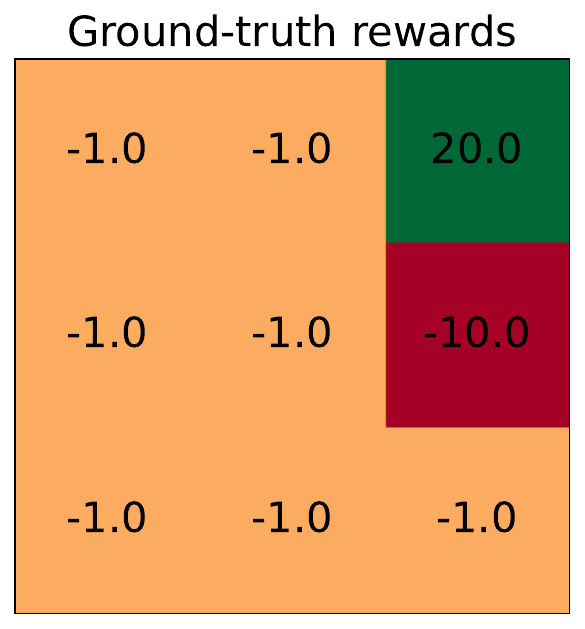} \; \includegraphics[width=0.35\textwidth]{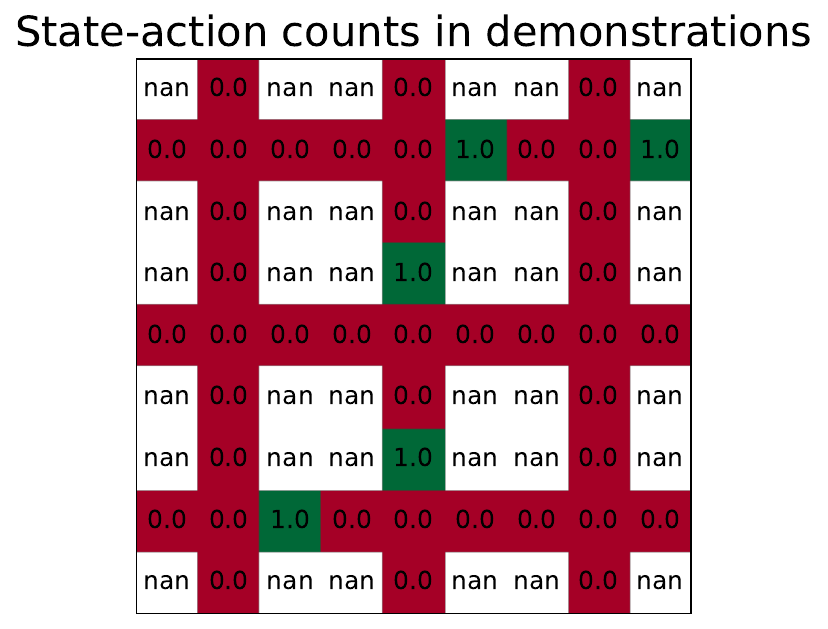} \; \includegraphics[width=0.32\textwidth]{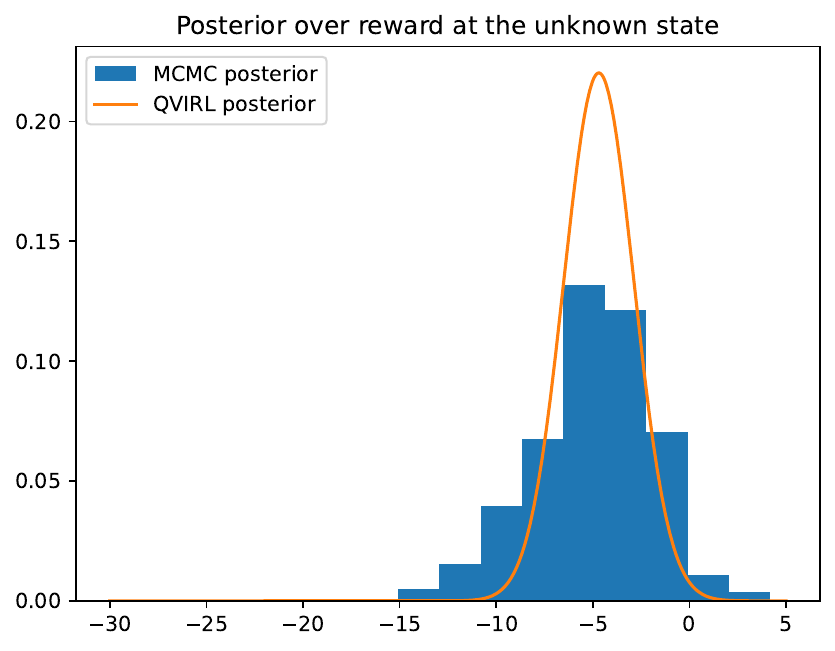}
    \caption{The ground truth rewards, demonstrations, and the posterior over the unknown right-middle reward recovered by QVIRL and Valuewalk in the illustrative gridworld.}
    \label{fig:3x3-gridworld-single-rew}
\end{figure}

For a human-readable and easily interpretable illustration of what our method is doing, we ran our method on the simple gridworld in Figure~\ref{fig:3x3-gridworld-single-rew} (left) where the top right state is terminal. We assume the learner knows all the rewards (which we model as a tight normal prior with std 0.1 centred at the true value) except for the right middle tile, where it has a normal prior with mean -1 and std of 10. We assume known deterministic dynamics, expert rationality coefficient $\beta=1.$, and $\gamma=0.8$.

We let the learner observe a single expert demonstration shown in the middle subplot. We are particularly interested in how an apprentice agent would behave in the bottom right cell -- would it be worth cutting through the unknown reward tile (to get to the goal faster), or should it go around? The expert demonstration would be consistent with either being optimal.

The right subplot shows the reward posterior recovered by QVIRL (as well as an MCMC reference posterior recovered using ValueWalk, which produces samples from the true posterior using MCMC). We find that if we use value iteration to find the apprentice policy maximizing the return with respect to the mean reward according to this posterior, the apprentice chooses to go through the unknown-reward tile (with Q-value of 8.0 for going up vs a Q-value of 5.2 of going to the left). On the other hand, if we solve for maximizing the mean minus 2 std of the posterior, representing a risk-averse policy, the Q-value of going up drops to 2.7 and the apprentice would choose to go around.

This illustrates how modelling posterior uncertainty allows us to behave in a risk-averse manner. Furthermore, the plot shows that the reward posterior deduced from the Q-value posterior fits well with the reference reward posterior.
\section{Code, Data, and Compute}
\subsection{Code and Data}
\label{app:code-and-data}
All experiments were run using publicly available libraries and data.
The main software components we used were Python 3.13 (available under the GPL-compatible PSF license agreement; https://docs.python.org/3/license.html), Farama Gymnasium 1.2.3 and Highway Env 1.10.2 (MIT License), and PyTorch 2.9.1 (BSD-3 license).

Our code will be made available on Github at https://github.com/bayesian-reward-learning/qvirl .


\subsection{Computing Resources Used}
\label{app:compute}
Except ATARI, each training was run on a single CPU core (of either an 8-core AMD Ryzen 7 PRO 7840U or Amazon Gravitron 8) and took between 2 and 15 minutes. Each ATARI training run used an NVIDIA V100 GPU and took about 90 minutes for the 200k steps on Pong, and around 7 hours for the 1M steps of Space Invaders.

\end{document}